\documentclass[aos,preprint,reqno]{imsart}

\setattribute{journal}{name}{}

\usepackage[utf8]{inputenc} 
\usepackage[T1]{fontenc}    
\usepackage[margin=1in]{geometry}
\usepackage{hyperref}       
\usepackage{url}            
\usepackage{booktabs}       
\usepackage{amsfonts}       
\usepackage{nicefrac}       
\usepackage{npbayes_mod} 
\usepackage{color}
\usepackage{mathtools}
\usepackage{tikz}
\usepackage{cleveref}
\usepackage{xspace}
\usepackage{xr}
\usepackage{dsfont}

\newtheorem{nlem}[theorem]{Lemma}
\crefname{nlem}{Lemma}{Lemmas}
\crefname{remark}{Remark}{Remarks}

\begin{document}

\begin{frontmatter}
\title{Edge-exchangeable graphs and sparsity}
\runtitle{Edge-exchangeable graphs and sparsity}

\begin{aug}
\author{\fnms{Diana}
\snm{Cai}
\ead[label=edc]{dcai@uchicago.edu},}
\author{\fnms{Trevor}
\snm{Campbell}
\ead[label=etc]{tdjc@mit.edu},}
\and
\author{\fnms{Tamara}
\snm{Broderick}
\ead[label=etb]{tbroderick@csail.mit.edu}}

\address[chiaddr]{Department of Statistics, \\
  University of Chicago, \\
  Chicago, IL, USA 60637
    \\ \printead{edc}
}
\address[csailaddr]{Computer Science and Artificial Intelligence Laboratory
  (CSAIL), \\
Massachusetts Institute of Technology, \\
  Cambridge, MA, USA 02139
    \\ \printead{etc,etb}
}
\runauthor{D.\ Cai, T.\ Campbell, T.\ Broderick}
\end{aug}

\begin{abstract}
Many popular network models rely on the assumption of
\emph{(vertex) exchangeability}, in which the distribution of the graph is invariant
to relabelings of the vertices.
However, the Aldous-Hoover theorem guarantees that
these graphs are dense or empty with probability one,
whereas many real-world graphs are sparse.
We present an alternative notion of
exchangeability for random graphs, which we call \emph{edge exchangeability},
in which the distribution of a graph sequence is invariant to
the order of the edges.
We demonstrate that edge-exchangeable models,
unlike models that are traditionally vertex exchangeable, can exhibit sparsity.
To do so, we outline a general framework for graph generative models;
by contrast to the pioneering work of \citet{caron2014arxiv},
models within our framework are stationary across steps of the graph sequence.
In particular, our model grows the graph by instantiating more latent atoms of a single random
measure as the dataset size increases, rather than adding new atoms to the
measure.
\end{abstract}


\end{frontmatter}

\maketitle

\section{Introduction}
In recent years, network data
have appeared in a growing number of applications, such
as online social networks, biological networks, and networks representing
communication patterns.
As a result, there is growing interest in developing models
for such data and studying their properties.
Crucially, individual network data sets also continue to increase in size;
we typically assume that the number of vertices is unbounded as time progresses.
We say a graph sequence is
\emph{dense} if
the number of edges grows quadratically in the number of vertices,
and a graph sequence is
\emph{sparse} if the number of edges grows
sub-quadratically as a function of the number of vertices.
Sparse graph sequences are more representative of real-world graph behavior.
However, many popular network models (see, e.g., \citet{DBLP:conf/nips/LloydOGR12} for an extensive list)
share the undesirable scaling property that
they yield dense sequences of graphs with probability one.
The poor scaling properties of these models can be traced back to
a seemingly innocent assumption: that the vertices in the model
are \emph{exchangeable}, that is, any finite permutation of the rows and columns of
the graph adjacency matrix does not change the distribution of the graph.
Under this assumption, the Aldous-Hoover theorem \citep{MR637937,Hoover79}
implies that such models generate dense or empty graphs with
probability one \citep{DBLP:journals/pami/OrbR14}.

This fundamental model misspecification motivates the development of new models that
can achieve sparsity.
One recent focus has been on models in which an additional parameter is employed to
uniformly decrease
the probabilities of edges
as the network grows
(e.g.,
\citet{MR2337396,BCCZ:arxiv,2013arXiv1309.5936W,borgs2015arxiv}).
While these models allow sparse graph sequences, the sequences are no
longer \emph{projective}. In projective sequences, vertices and edges are
added to a graph as a graph sequence progresses---whereas in the models above,
there is not generally any strict subgraph relationship between earlier graphs and later
graphs in the sequence.
Projectivity is natural in streaming modeling.
For instance, we may wish to capture new users joining a social network and new
connections being made among existing users---or new employees joining a company and
new communications between existing employees.

\citet{caron2014arxiv} have pioneered initial work on sparse, projective graph
sequences.
Instead of the \emph{vertex exchangeability} that yields the Aldous-Hoover theorem,
they consider a notion of
graph exchangeability based on the idea of independent increments of
subordinators \citep{MR2161313}, explored in depth by \citet{veitch:roy:arxiv}.
However, since this Kallenberg-style exchangeability
introduces a new countable infinity of latent vertices at every step in the graph sequence,
its generative mechanism seems particularly suited to the non-stationary domain.
By contrast, we are here interested in exploring \emph{stationary} models that grow in complexity with the size
of the data set. Consider classic Bayesian nonparametric models as the Chinese restaurant process (CRP)
and Indian buffet process (IBP); these engender growth by using a single infinite latent collection of parameters
to generate a finite but growing
set of instantiated parameters. Similarly, we propose a framework that uses
a single infinite latent collection of vertices to generate a finite but growing set of vertices
that participate in edges and thereby in the network.
We believe our framework will be a useful component in more complex,
non-stationary graphical models---just as the CRP and IBP are often combined with hidden Markov models or
other explicit non-stationary mechanisms.
Additionally, Kallenberg exchangeability is intimately tied to continuous-valued labels
of the vertices, and here we are interested in providing a characterization of the graph sequence
based solely on its topology.

In this work, we introduce a new form of exchangeability, distinct
from both vertex exchangeability and Kallenberg exchangeability.
In particular, we say that a graph sequence is \emph{edge exchangeable}
if the distribution of any graph in the sequence is
invariant to the \emph{order} in which edges arrive---rather than the order
of the vertices. We will demonstrate that edge exchangeability admits
a large family of sparse, projective graph
sequences.

In the remainder of the paper, we start by defining dense and sparse graph sequences
rigorously. We review vertex exchangeability before introducing our new notion of
edge exchangeability in \Cref{sec:exchangeability},
which we also contrast with Kallenberg exchangeability in more detail in \Cref{sec:relwork}.
We define a family of models, which we call \emph{graph frequency models},
based on random measures in \Cref{sec:generative}. We use these models to show that
edge-exchangeable models can yield sparse, projective graph sequences
via theoretical analysis
in \Cref{sec:poissp} and via simulations in \Cref{sec:simulations}.
Along the way, we highlight other benefits of the edge exchangeability and
graph frequency model frameworks.

\section{Exchangeability in graphs: old and new}
\label{sec:exchangeability}

Let
$(G_n)_n := G_1, G_2, \ldots$
be a sequence of graphs,
where each graph $G_n = (V_n, E_n)$ consists of a (finite) set of vertices $V_n$ and
a (finite) multiset of edges $E_n$.
Each edge $e \in E_n$ is a set of two vertices in $V_n$.
We assume the sequence is \emph{projective}---or growing---so that
$V_{n} \subseteq V_{n+1}$
and
$E_{n} \subseteq E_{n+1}$.
Consider, e.g., a social network with
more users joining the network and making new connections with existing users.
We say that a graph sequence is \emph{dense} if
$|E_n| = \Omega(|V_n|^2)$, i.e.,
the number of edges is asymptotically lower bounded by
$c \cdot |V_n|^2$ for some constant $c$.
Conversely, a sequence is \emph{sparse} if
$|E_n| = o(|V_n|^2)$, i.e., the number of edges is asymptotically upper bounded by
$c \cdot |V_n|^2$ for all constants $c$.
In what follows, we consider random graph sequences, and we focus on the case where $|V_n| \rightarrow \infty$
almost surely.

\subsection{Vertex-exchangeable graph sequences}

If the number of vertices in the graph sequence grows to infinity, the graphs in the sequence can be
thought of as subgraphs of an ``infinite'' graph with infinitely many vertices and a correspondingly infinite adjacency matrix.
Traditionally, exchangeability in random graphs is defined as the invariance of the distribution of any
finite submatrix of this adjacency matrix---corresponding to any finite collection of vertices---under finite permutation.
Equivalently, we can express this form of exchangeability, which we
henceforth call \emph{vertex exchangeability}, by considering a random sequence of graphs
$(G_n)_n$ with $V_n = [n]$, where $[n] := \{1,\ldots,n\}$. In this case, only the edge sequence is random.
Let $\pi$ be any permutation of the integers $[n]$. If $e = \{v,w\}$, let $\pi(e) := \{\pi(v),\pi(w)\}$.
If $E_n =  \{e_1,\ldots,e_m\}$, let $\pi(E_n) := \{\pi(e_1),\ldots,\pi(e_m)\}$.
\begin{definition}
Consider the
random graph sequence $(G_n)_n$, where $G_n$ has vertices $V_n = [n]$ and edges $E_n$. $(G_n)_n$ is (infinitely) \emph{vertex exchangeable}
if for every $n \in \Nats$ and
    for every permutation $\pi$ of the vertices $[n]$,
    $G_n \eqd \tilde{G}_n$, where $\tilde{G}_n$ has vertices $[n]$ and edges $\pi(E_n)$.
\end{definition}

A great many popular models for graphs are vertex exchangeable;
see \Cref{app:vertex}
and
\citet{DBLP:conf/nips/LloydOGR12} for a list.
However, it follows from
the Aldous-Hoover theorem \citep{MR637937,Hoover79} that
any vertex-exchangeable graph is a mixture of sampling procedures from \emph{graphons}.
Further, any graph sampled from a graphon is almost surely dense or empty \citep{DBLP:journals/pami/OrbR14}.
Thus,
vertex-exchangeable random graph models
are
misspecified models for sparse network datasets, as they
generate dense graphs.

\subsection{Edge-exchangeable graph sequences}

Vertex-exchangeable sequences have distributions invariant to the order of vertex
arrival. We introduce \emph{edge-exchangeable} graph sequences, which will instead be invariant to the
order of edge arrival. As before, we let $G_n = (V_n, E_n)$ be the $n$th graph in the sequence.
Here, though, we consider only \emph{active vertices}---that is, vertices that are connected via some edge.
That lets us define $V_n$ as a function of $E_n$; namely, $V_n$ is the union of the vertices in $E_n$. Note that a
graph that has sub-quadratic growth in the number of edges as a function of the number of active vertices will necessarily
have sub-quadratic growth in the number of edges as a function of the number of all vertices, so we obtain strictly stronger
results by considering active vertices. In this case, the graph $G_n$ is completely defined by its edge set $E_n$.

As above, we suppose that $E_{n} \subseteq E_{n+1}$. We can emphasize
this projectivity property by augmenting each edge with the step on which it is added to the sequence.
Let $E'_n$ be a collection of tuples, in which the first element is the edge and the
second element is the step (i.e., index) on which the edge is added: $E'_n = \{(e_1,s_1),\ldots,(e_m,s_m) \}$.
We can then define a \emph{step-augmented graph sequence} $(E'_n)_n = (E'_1,E'_2,\ldots)$
as a sequence of step-augmented edge sets. Note that there is a bijection between the step-augmented
graph sequence and the original graph sequence.

\newcommand{\rulesep}{\unskip\ \vrule\ }

\begin{figure}
    \centering
    \begin{subfigure}[b]{0.18\linewidth}
        \centering
        \begin{tikzpicture}
        \tikzstyle{vertex}=[circle,fill=blue!25,minimum size=17pt,inner sep=0pt]
            \node[vertex](A1) at (0,0) {1};
        \end{tikzpicture}
    \end{subfigure}
    \rulesep
    \begin{subfigure}[b]{0.18\linewidth}
        \centering
        \begin{tikzpicture}
        \tikzstyle{vertex}=[circle,fill=blue!25,minimum size=17pt,inner sep=0pt]
            \node[vertex](A1) at (0,0) {1};
            \node[vertex](A2) at (0,1.5) {2};
            \path[]
            (A1) edge node [left] {2} (A2);
        \end{tikzpicture}
    \end{subfigure}
    \rulesep
    \begin{subfigure}[b]{0.18\linewidth}
        \centering
        \begin{tikzpicture}
        \tikzstyle{vertex}=[circle,fill=blue!25,minimum size=17pt,inner sep=0pt]
            \node[vertex](A1) at (0,0) {1};
            \node[vertex](A2) at (0,1.5) {2};
            \node[vertex](A3) at (1.5,1.5) {3};
            \path[]
            (A1) edge node [left] {2} (A2);
        \end{tikzpicture}
    \end{subfigure}
    \rulesep
    \begin{subfigure}[b]{0.18\linewidth}
        \centering
        \begin{tikzpicture}
        \tikzstyle{vertex}=[circle,fill=blue!25,minimum size=17pt,inner sep=0pt]
            \node[vertex](A1) at (0,0) {1};
            \node[vertex](A2) at (0,1.5) {2};
            \node[vertex](A3) at (1.5,1.5) {3};
            \node[vertex](A4) at (1.5,0) {4};

            \path[]
            (A1) edge node [left] {2} (A2)
            (A4) edge node [above] {4} (A1)
            (A4) edge node [above] {4} (A2)
            (A4) edge node [right] {4} (A3);
        \end{tikzpicture}
    \end{subfigure}
    \rulesep \rulesep
    \begin{subfigure}[b]{0.18\linewidth}
        \centering
        \begin{tikzpicture}
        \tikzstyle{vertex}=[circle,fill=blue!25,minimum size=17pt,inner sep=0pt]
            \node[vertex](A1) at (0,0) {3};
            \node[vertex](A2) at (0,1.5) {2};
            \node[vertex](A3) at (1.5,1.5) {4};
            \node[vertex](A4) at (1.5,0) {1};

            \path[]
            (A1) edge node [left] {3} (A2)
            (A4) edge node [above] {3} (A1)
            (A4) edge node [above] {2} (A2)
            (A4) edge node [right] {4} (A3);
        \end{tikzpicture}
    \end{subfigure}

    \medskip
    \hrule
    \medskip

    \begin{subfigure}[b]{0.18\linewidth}
        \centering
        \begin{tikzpicture}
        \tikzstyle{vertex}=[circle,fill=blue!25,minimum size=17pt,inner sep=0pt]
            \node[vertex](A1) at (0,0) {2};
            \node[vertex](A2) at (0,1.5) {5};
            \path[]
            (A1) edge node [left] {1} (A2)
            (A2) edge [loop right] node {1} (A2);
        \end{tikzpicture}
    \end{subfigure}
    \rulesep
    \begin{subfigure}[b]{0.18\linewidth}
        \centering
        \begin{tikzpicture}
        \tikzstyle{vertex}=[circle,fill=blue!25,minimum size=17pt,inner sep=0pt]
            \node[vertex](A1) at (0,0) {2};
            \node[vertex](A2) at (0,1.5) {5};
            \path[]
            (A1) edge node [left] {1} (A2)
            (A2) edge [loop right] node {1} (A2);
        \end{tikzpicture}
    \end{subfigure}
    \rulesep
    \begin{subfigure}[b]{0.18\linewidth}
        \centering
        \begin{tikzpicture}
        \tikzstyle{vertex}=[circle,fill=blue!25,minimum size=17pt,inner sep=0pt]
            \node[vertex](A1) at (0,0) {2};
            \node[vertex](A2) at (0,1.5) {5};
            \path[]
            (A1) edge node [left] {1} (A2)
            (A1) edge [bend right=40, right] node {3} (A2)
            (A2) edge [loop right] node {1} (A2);
        \end{tikzpicture}
    \end{subfigure}
    \rulesep
    \begin{subfigure}[b]{0.18\linewidth}
        \centering
        \begin{tikzpicture}
        \tikzstyle{vertex}=[circle,fill=blue!25,minimum size=17pt,inner sep=0pt]
            \node[vertex](A1) at (0,0) {2};
            \node[vertex](A2) at (0,1.5) {5};
            \node[vertex](A3) at (1.5,1.5) {1};
            \node[vertex](A4) at (1.5,0) {6};

            \path[]
            (A1) edge node [left] {1} (A2)
            (A1) edge [bend right=40, right] node {3} (A2)
            (A2) edge [loop right] node {1} (A2)
            (A4) edge node [right] {4} (A3);
        \end{tikzpicture}
    \end{subfigure}
    \rulesep \rulesep
    \begin{subfigure}[b]{0.18\linewidth}
        \centering
        \begin{tikzpicture}
        \tikzstyle{vertex}=[circle,fill=blue!25,minimum size=17pt,inner sep=0pt]
            \node[vertex](A1) at (0,0) {2};
            \node[vertex](A2) at (0,1.5) {5};
            \node[vertex](A3) at (1.5,1.5) {1};
            \node[vertex](A4) at (1.5,0) {6};

            \path[]
            (A1) edge node [left] {4} (A2)
            (A1) edge [bend right=40, right] node {2} (A2)
            (A2) edge [loop right] node {4} (A2)
            (A4) edge node [right] {1} (A3);
        \end{tikzpicture}
    \end{subfigure}

    \caption{\emph{Upper, left four}: Step-augmented graph sequence from Ex.~\ref{ex:vertex_setup}. At each step $n$,
    the step value is always at least the maximum vertex index. \emph{Upper, right two}: Two graphs
    with the same probability under vertex exchangeability. \emph{Lower, left four}: Step-augmented graph sequence
    from Ex.~\ref{ex:step-aug}. \emph{Lower, right two}: Two graphs with the same probability under edge exchangeability. }
    \label{fig:ex-graphs}
\end{figure}
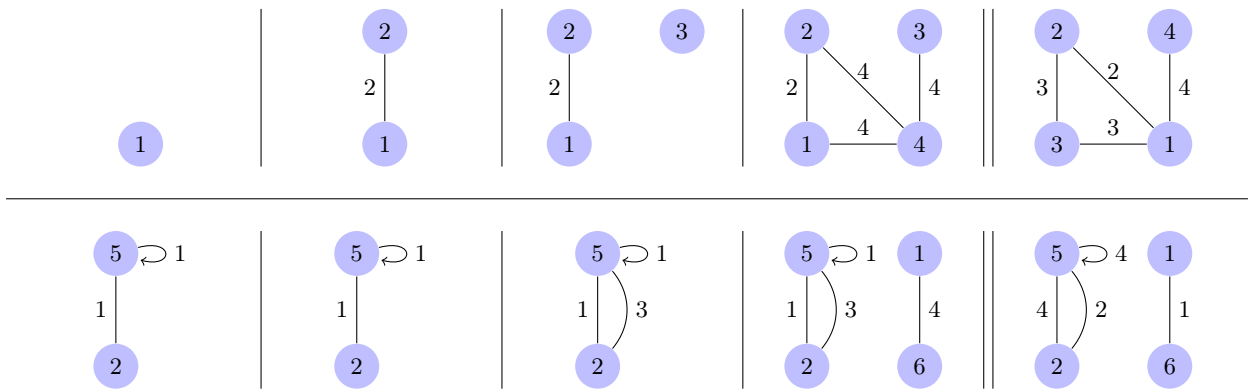

\begin{example}
\label{ex:vertex_setup}
In the setup for vertex exchangeability, we assumed $V_n = [n]$ and every edge is introduced
as soon as both of its vertices are introduced. In this case, the step of any edge
in the step-augmented graph is the maximum vertex value. For example,
in \Cref{fig:ex-graphs}, we have
\begin{align*}
    E'_1 &= \emptyset, 
    E'_2 = 
	E'_3 = \{(\{1,2\},2)\},
    E'_4 = \{(\{1,2\},2), (\{1,4\},4),(\{2,4\},4),(\{3,4\},4)\}.
\end{align*}
In general step-augmented graphs, though, the step need not equal the max
vertex, as we see next.
\end{example}

\begin{example}
\label{ex:step-aug}
   Suppose we have a graph given by the edge sequence (see \Cref{fig:ex-graphs}):
\begin{align*}
	E_1 &= E_2 = \{\{2,5\}, \{5,5\}\},
	E_3 = E_2 \cup \{\{2,5\} \},
	E_4 = E_3 \cup \{ \{1,6\} \}.
\end{align*}%
The step-augmented graph $E'_4$ is
$\{(\{2,5\},1),(\{5,5\},1), (\{2,5\},3), (\{1,6\},4)\}.$
\end{example}

Roughly, a random graph sequence is edge exchangeable if its distribution is invariant to finite permutations of the steps.
Let $\pi$ be a permutation of the integers $[n]$.
For a step-augmented edge set $E'_n = \{(e_1,s_1),\ldots,(e_m,s_m) \}$, let $\pi(E'_n) =
\{(e_1,\pi(s_1)),\ldots,(e_m,\pi(s_m)) \}$.

\begin{definition}
	Consider the random graph sequence $(G_n)_n$, where $G_n$ has step-augmented edges $E'_n$
	and $V_n$ are the active vertices of $E_n$. $(G_n)_n$ is (infinitely) \emph{edge exchangeable}
	if for every $n \in \Nats$ and for every permutation $\pi$ of the steps $[n]$,
	$G_n \eqd \tilde{G}_n$, where $\tilde{G}_n$ has step-augmented edges $\pi(E'_n)$ and
	associated active vertices.
\end{definition}
See \Cref{fig:ex-graphs} for visualizations of both vertex exchangeability and edge exchangeability.
 It remains to show that there are non-trivial models that are edge exchangeable
(\Cref{sec:generative}) and that edge-exchangeable models admit sparse graphs (\Cref{sec:poissp}).

\section{Edge-exchangeable graph frequency models}
\label{sec:generative}

We next demonstrate that a wide class of models, which we call \emph{graph frequency models}, exhibit edge exchangeability.
Consider a latent infinity of vertices indexed by the positive integers $\Nats = \{1,2,\ldots\}$,
along with an infinity of edge labels $(\theta_{\{i,j\}})$, each in a
set $\Theta$, and positive edge rates (or frequencies) $(w_{\{i,j\}})$ in $\mathbb{R}_+$.
We allow both the $(\theta_{\{i,j\}})$ and $(w_{\{i,j\}})$ to be random, though this is not mandatory.
For instance, we might choose $\theta_{\{i,j\}} = (i,j)$ for $i \le j$, and $\Theta = \mathbb{R}^2$.
Alternatively, the $\theta_{\{i,j\}}$ could be drawn iid from a continuous distribution such as $\textrm{Unif}[0,1]$.
For any choice of $(\theta_{\{i,j\}})$ and $(w_{\{i,j\}})$,
\begin{align}
	\label{eq:freqm}
	W := \sum_{\{i,j\}: i,j \in \Nats} w_{\{i,j\}} \delta_{\theta_{\{i,j\}}}
\end{align}
is a \emph{measure} on $\Theta$.
Moreover, it is a discrete measure since it is always atomic.
If either $(\theta_{\{i,j\}})$ or $(w_{\{i,j\}})$ (or both)
are random, $W$ is a \emph{discrete random measure} on $\Theta$
since it is a random, discrete-measure-valued element.
Given the edge rates (or frequencies) $(w_{\{i,j\}})$ in $W$, we next show some natural ways to construct edge-exchangeable graphs.

\paragraph{Single edge per step}
If the rates $(w_{\{i,j\}})$ are normalized such that $\sum_{\{i,j\}: i, j \in \Nats} w_{\{i,j\}} = 1$,
then $(w_{\{i,j\}})$ is a distribution over all possible vertex pairs. In other words, $W$ is a
probability measure.
We can form an edge-exchangeable graph sequence by first drawing values for
$(w_{\{i,j\}})$ and $(\theta_{\{i,j\}})$---and setting $E_0 = \emptyset$. We recursively
set $E_{n+1} = E_n \cup \{ e \}$, where $e$ is an edge $\{i,j\}$ chosen from the distribution
$(w_{\{i,j\}})$. This construction introduces a single edge in the graph each step, although it may
be a duplicate of an edge that already exists. Therefore, this technique generates multigraphs
one edge at a time.
Since the edge every step is drawn conditionally iid given $W$, we have an edge-exchangeable graph.

\paragraph{Multiple edges per step}
Alternatively, the rates $(w_{\{i,j\}})$
may not be normalized.
Then $W$ may not be a probability measure.
Let $f(m | w)$ be a distribution over non-negative integers $m$ given some rate
$w \in \mathbb{R}_{+}$.
We again initialize our sequence by drawing
$(w_{\{i,j\}})$ and $(\theta_{\{i,j\}})$ and setting $E_0 = \emptyset$.
In this case, recursively, on the $n$th step,
start by setting $F = \emptyset$.
For every possible edge $e = \{i,j\}$, we draw the multiplicity of the edge $e$
in this step as $m_{e} \indep f(\cdot | w_{e})$
and add $m_{e}$ copies of edge $e$ to $F$. Finally, $E_{n+1} = E_n \cup F$.
This technique potentially introduces multiple edges in each step, in which
edges themselves may have multiplicity greater than one and may be duplicates of edges that already exist in the graph.
Therefore, this technique generates multigraphs, multiple edges at a time.
If we restrict $f$ and $W$ such that finitely many edges are added on every step almost surely,
we have an edge-exchangeable graph, as the edges in each step are drawn conditionally iid given $W$.

Given a sequence of edge sets $E_0, E_1, \dots$ constructed via either of the above methods, we can form a binary graph sequence
$\bar E_0, \bar E_1, \dots$ by setting $\bar E_i$ to have the same edges as $E_i$ except with multiplicity $1$. Although
this binary graph
is not itself edge exchangeable, it inherits many of the properties (such as sparsity, as shown in \Cref{sec:poissp}) of the underlying edge-exchangeable multigraph.

The choice of the distribution on the measure $W$
has a strong influence on the properties of the resulting edge-exchangeable graph
sampled via one of the above methods.
For example, one choice is to set $w_{\{i,j\}} = w_i w_j$, where the $(w_i)_i$ are a countable infinity of random values generated
according to a \emph{Poisson point process} (PPP). We say that $(w_i)_i$ is distributed
according to a Poisson point process parameterized by rate measure $\nu$,
$(w_i)_i \sim \textrm{PPP}(\nu)$,
if (a) $\#\{i: w_i \in A\} \sim \textrm{Poisson}(\nu(A))$ for any set $A$ with finite measure $\nu(A)$ and (b)
$\#\{i: w_i \in A_j\}$ are independent random variables across any finite collection of disjoint sets $(A_j)_{j=1}^{J}$.
In \Cref{sec:poissp} we examine a particular example of this graph frequency model, and demonstrate
that sparsity is possible in edge-exchangeable graphs.

\section{Related work and connection to nonparametric Bayes}
\label{sec:relwork}

Given a unique label $\theta_i$ for each vertex $i\in\mathbb{N}$,
and denoting $g_{ij} = g_{ji}$ to be the number of undirected edges between vertices $i$ and $j$,
the graph itself can be represented as the discrete random measure
$
    G = \sum_{i,j}
    g_{ij}
    \delta_{(\theta_i, \theta_j)}
$ on $\Reals_+^2$.
A different notion of exchangeability for graphs than
the ones in \Cref{sec:exchangeability}
can be phrased for such atomic random measures:
a point process $G$ on $\Rplus^2$ is (jointly) exchangeable if,
for all finite permutations $\pi$ of $\Nats$ and all $h > 0$,
\begin{align*}
    G(A_i \times A_j) \eqd   G(A_{\pi(i)} \times A_{\pi(j)}),
    \text{~for~} (i,j) \in \Nats^2,
    \qquad \text{where~} A_i :=[h\cdot(i-1), h\cdot i].
\end{align*}
This form of exchangeability, which we refer to as \emph{Kallenberg exchangeability},
can intuitively be viewed as invariance of the graph distribution to relabeling of the vertices,
which are now embedded in $\Rplus^2$. As such it is analogous to vertex exchangeability, but for
discrete random measures \citep[Sec.~4.1]{caron2014arxiv}.
Exchangeability for random measures was introduced by
Aldous \citep{MR883646}, and a representation theorem was given by
Kallenberg \citep[Ch.\ 9]{MR2161313,MR1031426}.
The use of Kallenberg exchangeability for modeling graphs was
first proposed by
\citet{caron2014arxiv},
and then characterized in greater generality by
\citet{veitch:roy:arxiv}
and
\citet{BCCH:arxiv}.
Edge exchangeability is distinct from Kallenberg exchangeability,
as shown by the following example.
\begin{example}[Edge exchangeable but not Kallenberg exchangeable]
    Consider the graph frequency model developed in \Cref{sec:generative}, with $w_{\{i,j\}} = (ij)^{-2}$ and $\theta_{\{i,j\}} = \{i,j\}$.
    Since the edges at each step are drawn iid given $W$, the graph sequence is edge exchangeable. However,
    the corresponding graph measure $G = \sum_{i,j} n_{ij}\delta_{(i, j)}$ (where $n_{ij} = n_{ji} \sim\mathrm{Binom}(N, (ij)^{-2})$)
    is not Kallenberg exchangeable, since the probability of generating edge $\{i,j\}$ is directly related to the positions $(i,j)$ and $(j, i)$ in $\mathbb{R}_+^2$ of the corresponding atoms
    in $G$ (in particular, the probability is decreasing in $ij$).
\end{example}

\begin{figure}
    \centering
    \begin{subfigure}[b]{0.49\linewidth}
        \centering
        \includegraphics[scale=0.42]{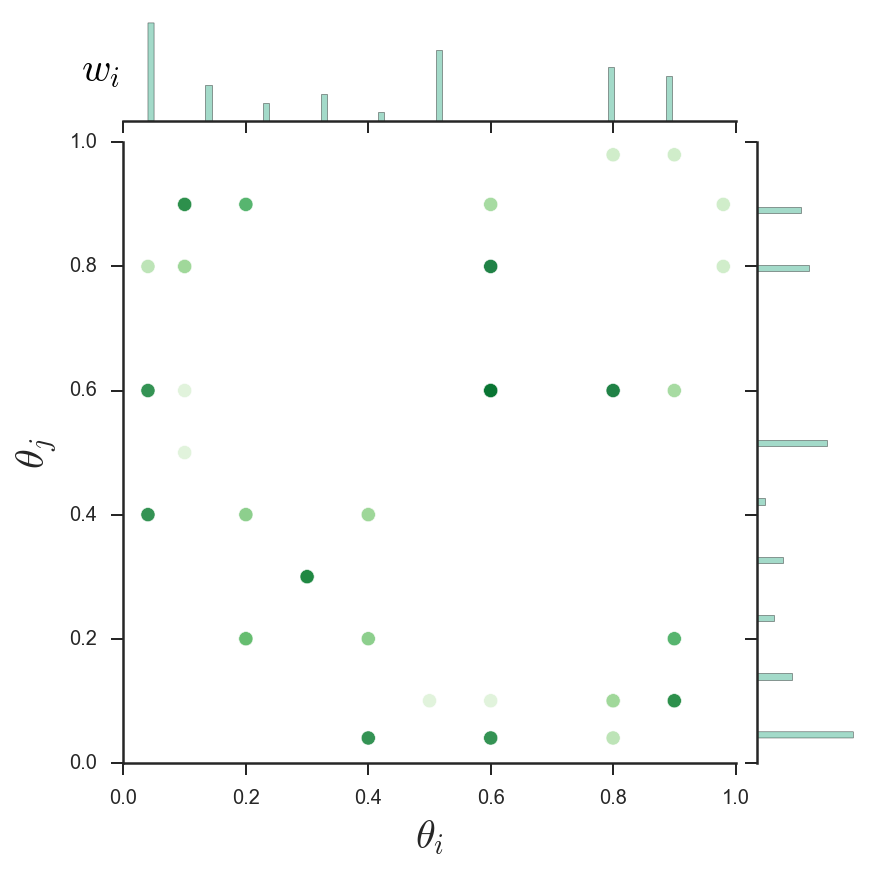}
        \caption{Graph frequency model (fixed $y$, $n$ steps)}
    \end{subfigure}
    \begin{subfigure}[b]{0.49\linewidth}
        \centering
        \includegraphics[scale=0.42]{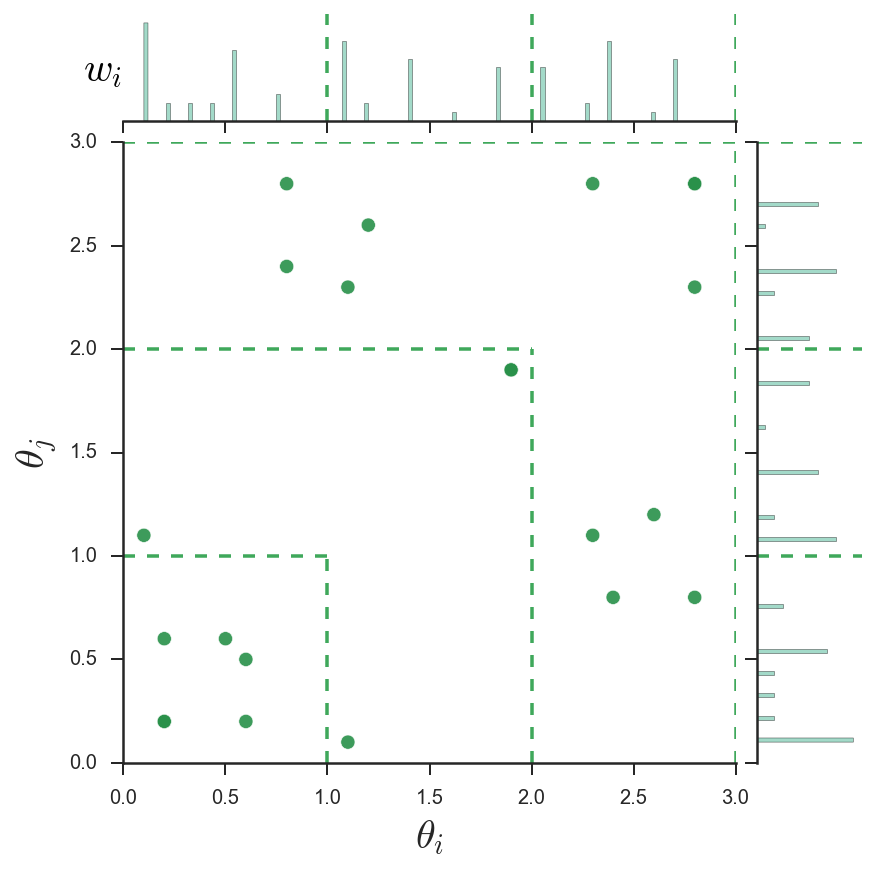}
        \caption{Caron--Fox, PPP on $[0,y]\times [0,y]$ (1 step, $y$ grows)}
    \end{subfigure}
    \caption{
    	A comparison of a graph frequency model (\Cref{sec:generative} and \Cref{eq:graphmodel})
	and the generative model of \citet{caron2014arxiv}.
	Any interval $[0,y]$ contains a countably
        infinite number of atoms with a nonzero weight in the random measure; a draw
        from the random measure is plotted at the top (and repeated on the right side).
        Each atom corresponds
        to a latent vertex.
        Each point $(\theta_i,\theta_j)$ corresponds to a latent edge. Darker point colors
        on the left occur for greater edge multiplicities. On the \emph{left}, more latent edges are instantiated
        as more steps $n$ are taken. On the \emph{right}, the edges within $[0,y]^2$ are fixed, but more edges
        are instantiated as $y$ grows.
    }
    \label{fig:cffreq}
    \vspace{-10pt}
\end{figure}

Our graph frequency model is reminiscent of the \citet{caron2014arxiv} generative model, but has a number of key differences.
At a high level, this earlier model generates a weight measure $W = \sum_{i,j} w_{ij}\delta_{(\theta_i,\theta_j)}$
(\citet{caron2014arxiv} used, in particular, the outer product of a completely random measure),
and the graph measure $G$ is constructed by sampling $g_{ij}$ \emph{once} given $w_{ij}$ for each pair $i, j$.
To create a finite graph, the graph measure $G$ is restricted to the
subset $[0, y]\times [0, y]\subset\mathbb{R}_+^2$ for $0<y<\infty$; to create
a projective growing graph sequence, the value of $y$ is increased.
By contrast, in the analogous graph frequency model of the present work, $y$
is fixed, and we grow the network by \emph{repeatedly} sampling the number of edges $g_{ij}$ between vertices $i$ and $j$
and summing the result.
Thus, in the \citet{caron2014arxiv} model, a latent infinity of vertices (only finitely many of which are active) are added to the network each time $y$ increases.
In our graph frequency model, there is a single collection of latent vertices, which
are all gradually activated by increasing the number of samples that generate edges
between the vertices. See \Cref{fig:cffreq} for an illustration.

Increasing $n$ in the graph frequency model has the interpretation of both
(a) time passing and (b) new individuals joining a network because they have formed a connection
that was not previously there. In particular, only latent individuals that will eventually join the network
are considered.
This behavior is analogous to the well-known behavior of other nonparametric Bayesian models
such as, e.g., a Chinese
restaurant process (CRP). In this analogy, the Dirichlet process (DP) corresponds to our
graph frequency model, and the clusters instantiated by the CRP correspond to the vertices that are active after $n$ steps.
In the DP, only latent clusters that will eventually appear in the data are modeled.
Since the graph frequency setting is stationary like the DP/CRP, it may be more straightforward
to develop approximate Bayesian inference algorithms, e.g., via truncation \citep{Campbell16}.

Edge exchangeability first appeared in work by \citet{crane:dempsey15,
crane:dempsey15b,williamson16},
and \citet{Broderick15, Broderick15b,
Cai15}.
\citet{Broderick15, Broderick15b} established the notion of edge exchangeability used here and
provided characterizations via exchangeable partitions and feature allocations, as in \Cref{app:edge}.
\citet{Broderick15,Cai15}
developed a frequency model based on weights $(w_i)_{i}$ generated from a
Poisson process and studied several types of power laws in the model.
\citet{crane:dempsey15} established a similar
notion of edge exchangeability in the context of a larger statistical modeling framework.
\citet{crane:dempsey15b,crane:dempsey15} provided sparsity and power law results for the case where the weights
$(w_i)_{i}$ are generated from a Pitman-Yor process and
power law degree distribution simulations.
\citet{williamson16} described a similar notion of edge exchangeability and developed an edge-exchangeable model where
the weights $(w_i)_i$ are generated from a Dirichlet process,
a mixture model extension,
and an efficient Bayesian inference procedure.
In work concurrent to the present paper, \citet{crane:dempsey16} re-examined edge exchangeability, provided a representation theorem,
and studied sparsity and power laws for the same model based on Pitman-Yor weights.
By contrast, we here obtain sparsity
results across all Poisson point process-based graph frequency models of the form in \Cref{eq:graphmodel} below, and use a specific
three-parameter beta process rate measure only for simulations in \Cref{sec:simulations}.

\section{Sparsity in Poisson process graph frequency models}
\label{sec:poissp}
We now demonstrate that, unlike vertex exchangeability, edge exchangeability allows for sparsity
in random graph sequences.
We develop a class of sparse, edge-exchangeable multigraph
sequences
via the Poisson point process construction introduced in \Cref{sec:generative}, along with
their binary restrictions.

\paragraph{Model}
Let $\mathcal{W}$ be a Poisson process on $[0, 1]$ with
a nonatomic, $\sigma$-finite rate measure $\nu$ satisfying
$\nu([0, 1]) = \infty$ and $\int_0^1 w \nu(\mathrm{d}w) < \infty$.
These two conditions on $\nu$ guarantee that $\mathcal{W}$ is a countably infinite collection of rates in $[0,1]$
and that $\sum_{w\in\mathcal{W}} w< \infty$ almost surely.
We can use $\mathcal{W}$ to construct the set of rates:
$w_{\{i,j\}} = w_iw_j$ if $i\neq j$, and $w_{\{i, i\}}=0$. The edge labels $\theta_{\{i,j\}}$ are unimportant in characterizing
sparsity, and so can be ignored.

To use the multiple-edges-per-step graph frequency model from \Cref{sec:generative},
we let $f(\cdot | w)$ be Bernoulli with probability $w$.
Since edge
$\{i, j\}$ is added in each step with probability $w_iw_j$, its multiplicity $M_{\{i,j\}}$
after $n$ steps
has a binomial distribution with parameters $n, w_iw_j$.
Note that self-loops are avoided by setting $w_{\{i, i\}} =0$.
Therefore, the graph after $n$ steps is described by:
\begin{align}
    \label{eq:graphmodel}
   \mathcal{W} \sim \text{PPP}(\nu)
   && M_{\{i,j\}} \indep \text{Binom}(n, w_iw_j) \, \, \text{~for~}
    i < j \in\mathbb{N}.
\end{align}
As mentioned earlier, this generative model yields an edge-exchangeable graph,
with edge multiset $E_n$ containing $\{i, j\}$ with multiplicity $M_{\{i, j\}}$, and active vertices $V_n = \{i: \sum_j M_{\{i, j\}}>0\}$.
Although this model generates multigraphs, it can be modified to sample a binary
graph $(\bar V_n, \bar E_n)$ by setting
$\bar V_n = V_n$ and $\bar E_n$ to the set of edges $\{i, j\}$ such that $\{i, j\}$ has multiplicity $\geq 1$ in $E_n$.
We can express the number of vertices and edges, in the multi- and binary graphs respectively, as
\begin{align*}
\hspace{-.25cm}|\bar V_n|\! =\! |V_n| \!=\! \sum_{i} \mathds{1}\!\!\left(\!\sum_{j\neq i} M_{\{i,j\}} > 0\!\right) , \,
\,\,\, |E_n| = \frac{1}{2}\sum_{i\neq j}M_{\{i, j\}}, \,\,\,
|\bar E_n| = \frac{1}{2}\sum_{i\neq j} \mathds{1}\left(M_{\{i, j\}}>0\right).
\end{align*}

\paragraph{Moments}
Recall that a sequence of graphs is considered \emph{sparse} if $|E_n| = o(|V_n|^2)$.
Thus, sparsity in the present setting is an \emph{asymptotic} property of a random graph sequence.
Rather than consider the asymptotics of the (dependent) random sequences $|E_n|$ and $|V_n|$ in concert,
\Cref{lem:xvsex} allows us to consider the asymptotics of their first moments, which are deterministic sequences
and can be analyzed separately. We use $\sim$ to denote asymptotic equivalence,
i.e., $a_n\sim b_n \iff \lim_{n\to\infty} \frac{a_n}{b_n} = 1$.
For details on our asymptotic notation and proofs for this section, see \Cref{app:proofs}.
\begin{nlem}\label{lem:xvsex}
The number of vertices and edges for both the multi- and binary graphs satisfy
\begin{align*}
 |\bar V_n| = |V_n| \overset{\text{a.s.}}{\sim} \mbe\left(|V_n|\right), \qquad |E_n| \overset{\text{a.s.}}{\sim} \mbe\left(|E_n|\right), \qquad |\bar E_n| \overset{\text{a.s.}}{\sim} \mbe\left(|\bar E_n|\right), \qquad n\to \infty.
\end{align*}
\end{nlem}
Thus, we can examine the asymptotic behavior of the random numbers of edges and vertices by examining
the asymptotic behavior of their expectations, which are provided by \Cref{prop:moments}.
\begin{nlem}\label{prop:moments}
The expected numbers of vertices and edges for the multi- and binary graphs are
\begin{align*}
\mbe\left(|\bar V_n|\right) = \mbe\left(|V_n|\right)
          &= \int \left[1 - \exp\left(-\int (1- (1-wv)^n)
          \nu(\mathrm{d}v)\right) \right]
        \nu(\mathrm{d}w),\\
\mbe\left(|E_n|\right) = \frac{n}{2}\iint wv \, \nu(\mathrm{d}w)\nu(\mathrm{d}v),
    &\qquad   \mbe\left(|\bar E_n|\right) = \frac{1}{2}\iint (1 - (1-wv)^n)
    \, \nu(\mathrm{d}w)\nu(\mathrm{d}v).
\end{align*}
\end{nlem}

\paragraph{Sparsity}
We are now equipped to characterize the sparsity of this random graph sequence:
\begin{theorem}\label{thm:asymp}
Suppose $\nu$ has a regularly varying tail, i.e.,
there exist $\alpha \in (0, 1)$ and $\ell : \mathbb{R}_+ \to \mathbb{R}_+$
s.t.
\begin{align*}
\int_x^1\nu(\mathrm{d} w) &\sim x^{-\alpha}\ell(x^{-1}), \quad x\to 0
\qquad\text{and}\qquad
\forall c > 0,\,\,  \lim_{x\rightarrow\infty}
\frac{\ell(cx)}{\ell(x)} = 1.
\end{align*}
Then as $n\to\infty$,
\begin{align*}
\hspace{-.2cm}|V_n| \eqas \Theta(n^\alpha\ell(n)),
    && |E_n| \eqas \Theta(n),
    && |\bar E_n|
    \eqas O\left(\ell({n}^{1/2})\min\left(n^{\frac{1+\alpha}{2}},\ell(n)n^{\frac{3\alpha}{2}}\right)\right).
\end{align*}
\end{theorem}
\Cref{thm:asymp} implies that the multigraph is sparse when $\alpha \in (\nicefrac{1}{2}, 1)$, and that
the restriction to the binary graph is sparse for any $\alpha \in (0, 1)$.
See Remark~\ref{proof-remark} for a discussion.
Thus, edge-exchangeable random graph sequences
allow for a wide range of sparse and dense behavior.

\section{Simulations}
\label{sec:simulations}

In this section, we explore the behavior of graphs generated by the model from \Cref{sec:poissp} via simulation,
with the primary goal of empirically demonstrating that the model produces sparse graphs.
We consider the case when the Poisson process generating the weights in \Cref{eq:graphmodel} has the rate measure of
a \emph{three-parameter beta process} (3-BP)
on $(0,1)$ \citep{teh2009indian, MR2934958}:
\begin{align}
    \nu(dw) =
    \gamma \frac{\Gamma(1+\beta)}{\Gamma(1-\alpha)\Gamma(\alpha+\beta)} w^{-1-\alpha}
    (1 - w)^{\alpha+\beta-1}\, dw,
\end{align}
with mass $\gamma > 0$, concentration $\beta > 0$, and discount $\alpha \in (0, 1)$.
In order for the 3-BP to have finite total mass $\sum_j w_j < \infty$,
we require that $\beta > -\alpha$.
We draw realizations of the weights from a $\text{3-BP}(\gamma,\beta,\alpha)$ according to the
stick-breaking representation given by \citet*{MR2934958}.
That is, the $w_i$ are the atom weights of the measure ${W}$
for
\begin{align*}
{{W}}  &= \sum_{i=1}^\infty \sum_{j=1}^{C_i} V_{i,j}^{(i)}
\prod_{l=1}^{i-1} (1 - V_{i,j}^{(\ell)})\delta_{\psi_{i,j}},  &
C_i               &\iid \text{Pois}(\gamma), \\
V_{i,j}^{(\ell)}  &\indep \text{Beta}(1-\alpha, \beta + \ell\alpha), &
\psi_{i,j}        &\iid  B_0
\end{align*}
and any continuous (i.e., non-atomic) choice of distribution $B_0$.

Since simulating an infinite number of atoms is not possible,
we truncate the outer summation in $i$ to 2000
rounds, resulting in $\sum_{i=1}^{2000} C_i$ weights.
The parameters of the beta process were fixed to $\gamma=3$ and $\theta=1$, as they do not influence
the sparsity of the resulting graph frequency model, and
we varied the discount parameter $\alpha$.
Given a single draw {${W}$} (at some specific discount $\alpha$),
we then simulated the edges of the graph,
where the number of Bernoulli draws $N$ varied between 50 and 2000.

\Cref{fig:type1multi} shows how the number of edges varies versus
the total number of active vertices for the multigraph, with different
colors representing different
random seeds.
To check whether the generated graph was sparse,
we determined the exponent
by examining the slope of the data points
(on a log-scale).
In all plots, the black dashed line is a line with
slope 2.
In the multigraph, we found that for
the discount parameter settings $\alpha = 0.6, 0.7$,
the slopes were below 2;
for $\alpha = 0, 0.3$, the slopes
were greater than 2.
This corresponds to our theoretical results; for $\alpha < 0.5$ the multigraph is dense with slope greater than 2,
and for $\alpha > 0.5$ the multigraph is sparse with slope less than 2.
Furthermore, the sparse graphs exhibit \emph{power law} relationships between the number
of edges and vertices, i.e.,
$|{E}_N| \simas c\, |{V}_N|^b, \, N \rightarrow \infty$, where $b \in
(1,2)$,
as suggested by the linear relationship
in the plots
between the quantities on a log-scale.
Note that there are necessarily fewer edges in the binary graph than in the multigraph,
and thus this plot implies that the binary graph frequency model can also capture sparsity.
\Cref{fig:type1bin} confirms this observation; it shows how the number of edges
varies with the number of active vertices for the binary graph.
In this case, across $\alpha \in (0,1)$,
we observe slopes
that are less than 2.
This agrees with our theory from  \Cref{sec:poissp}, which states that the binary graph is sparse
for any $\alpha \in (0, 1)$.

\begin{figure}
    \centering
    \begin{subfigure}[b]{\linewidth}
        \centering
        \includegraphics[scale=0.27]{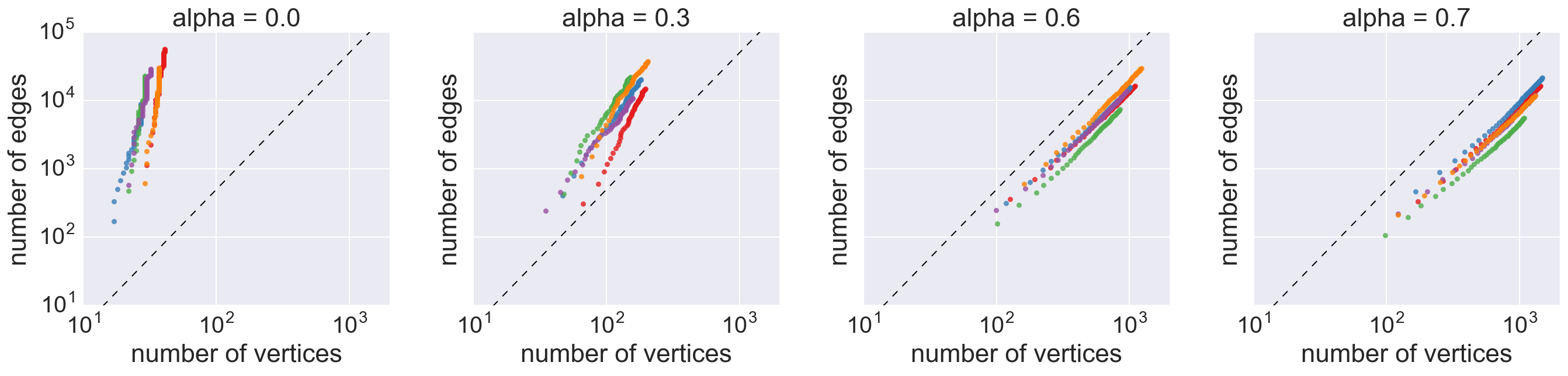}
        \caption{Multigraph edges vs.~active vertices}\label{fig:type1multi}
    \end{subfigure}
    \begin{subfigure}[b]{\linewidth}
        \centering
        \includegraphics[scale=0.27]{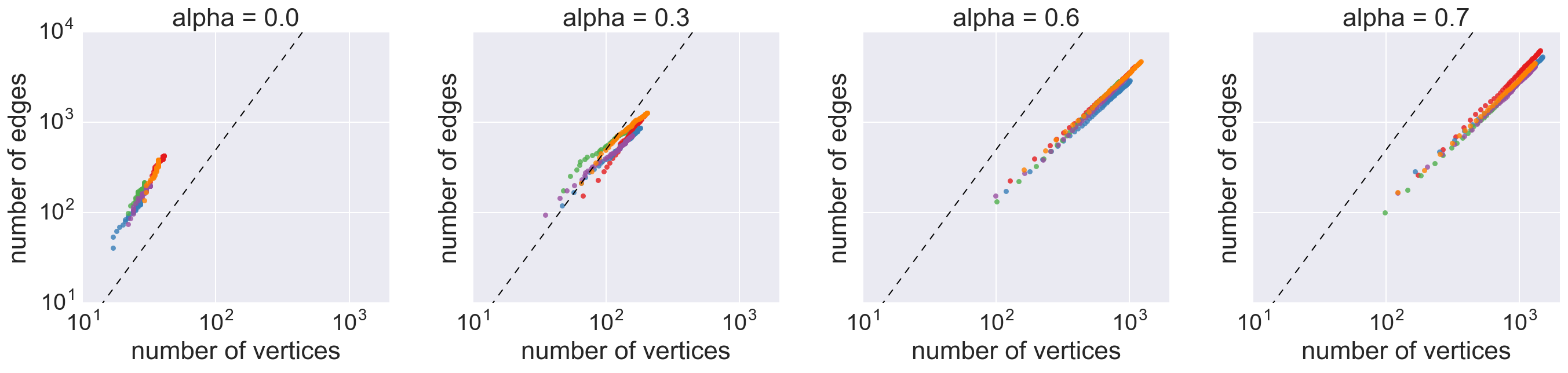}
        \caption{Binary graph edges vs.~active vertices}\label{fig:type1bin}
    \end{subfigure}
    \caption{
        Data simulated from a
        graph frequency model with weights generated
        according to a 3-BP.
        Colors represent different random draws.
        The dashed line has a slope of 2.
        }\label{fig:sims}
        \vspace{-15pt}
\end{figure}

\section{Conclusions}
We have proposed an alternative form of
exchangeability for random graphs, which we call \emph{edge exchangeability},
in which the distribution of a graph sequence is invariant to
the order of the edges. We have demonstrated that edge-exchangeable graph sequences, 
unlike traditional vertex-exchangeable sequences,
can be sparse by developing a class of 
edge-exchangeable graph frequency models that provably exhibit sparsity.
Simulations using edge frequencies drawn according to a
three-parameter beta process confirm our theoretical
results regarding sparsity. Our results suggest that a variety of future directions
would be fruitful---including theoretically
characterizing different types of power laws within graph frequency models,
characterizing the use of truncation within
graph frequency models as a means for approximate Bayesian inference
in graphs, and understanding the full range of distributions over sparse, edge-exchangeable
graph sequences.

\subsubsection*{Acknowledgments}
We would like to thank
Bailey Fosdick and
Tyler McCormick
for helpful conversations.

\appendix
\section{Overview}
In \Cref{app:vertex},
we provide more examples of graph models that are either vertex exchangeable
or Kallenberg exchangeable.
In \Cref{app:edge}, we establish characterizations of
edge exchangeability in graphs
via existing notions of exchangeability
for combinatorial structures such as random partitions
and feature allocations.
In \Cref{app:proofs}, we provide full proof details for the theoretical results
in the main text.

\section{More exchangeable graph models}
\label{app:vertex}

Many popular graph models are vertex exchangeable. These models
include the classic Erd\H{o}s--R\'enyi model \citep{MR0120167},
as well as Bayesian generative models for network data, such as
the stochastic block model \citep{MR718088},
the mixed membership
stochastic block model \citep{airoldi2008mixed},
the infinite relational model
\citep{C.Kemp:2006:53fd9, DBLP:conf/mlg/XuTYYK07},
the latent space model \citep{MR1951262},
the latent feature relational model \citep{DBLP:conf/nips/MillerGJ09},
the infinite latent attribute model \citep{DBLP:conf/icml/PallaKG12},
and the random function model \citep{DBLP:conf/nips/LloydOGR12}.
See
\citet{DBLP:journals/pami/OrbR14} and
\citet{DBLP:conf/nips/LloydOGR12} for more examples and discussion.

Recently, a number of extensions to the Kallenberg-exchangeable model
of
\citet{caron2014arxiv},
which builds on early work on bipartite
graphs by \citet{DBLP:conf/nips/Caron12},
have also been developed. These models include
extensions to stochastic block models \citep{herlau2015},
mixed membership stochastic block models \citep{todeschini2016arxiv},
and dynamic network models \citep{palla2016arxiv}.

\section{Characterizations of edge-exchangeable graph sequences}
\label{app:edge}

We introduced edge exchangeability,
a new notion of exchangeability for graphs.
Just as the Aldous-Hoover theorem provides
a characterization of the
distribution of vertex-exchangeable graphs,
it is desirable to provide a characterization
of edge exchangeability in graphs.
Below we show how characterization theorems that already
exist for other combinatorial structures can be readily applied
to provide characterizations for edge exchangeability in graphs.

We first develop mappings from edge-exchangeable graph sequences
to familiar combinatorial structures---such as partitions \citep{pitman:1995:exchangeable},
feature allocations \citep{broderick:2013:feature}, and trait allocations \citep{broderick:2014:posteriors,Campbell16b}---showing
that edge exchangeability in the graph corresponds to exchangeability in those structures.
In this manner, we
provide characterizations of the case
where one edge is added to the graph per step in \Cref{sec:part},
where multiple unique edges may be added per step in \Cref{sec:feat},
and where multiple (non)unique edges may be added in \Cref{sec:trait}.

A limitation of these connections is that it is not immediately clear how to recover
the connectivity in the graph from the mapped combinatorial object; for instance,
given a particular feature allocation, the graph to which it corresponds is not identifiable.
This issue has been addressed in a purely combinatorial context via \emph{vertex allocations}
and the \emph{graph paintbox} \citep{Campbell16b} using the general theory of trait allocations.
In \Cref{sec:labeledcomb}, we provide an alternative connection to \emph{ordered} combinatorial
structures \citep{broderick:2013:feature,Campbell16b} under the assumption that vertex labels are provided.
This assumption is often reasonable in the setting of network data where the vertices and edges are observed directly.
By contrast, it is unusual to assume that labels are provided for blocks in the case of partitions,
feature allocations, and trait allocations since, in these cases, the combinatorial structure is typically entirely latent in real data analysis
problems. For instance, in clustering applications, finding parameters that describe each cluster is usually part of the inference problem. In the graph case, though,
the use of an ordered structure identifies the particular pair of vertices corresponding to each edge in the graph, allowing
recovery of the graph itself.

\subsection{The step collection sequence and connections to other forms of
combinatorial exchangeability} \label{app:step_connections}

In order to analyze edge-exchangeable graphs using
the existing combinatorial machinery of
random
partitions, feature allocations, and trait allocations,
we introduce a new combinatorial structure,
the step collection sequence,
which can take the form of a sequence of partitions,
feature allocations, or trait allocations.
As we will now see, the step collection sequence can be
constructed from the step-augmented graph
sequence in the following way.

Suppose we assign a unique label $\phi$ to each pair of vertices.
Then if a pair of vertices is labeled $\phi$,
we may imagine that any particular edge between this pair of vertices
is assigned label $\phi$ when it appears. Let $\phi_j$ be the $j$th such
unique
edge label.

Recall that we consider a sequence of graphs defined by its
step-augmented edge sequence $E'_n$.
Let $S_j$ be the set of steps up to the current step $n$ in which
any edge labeled $\phi_j$ was added. If $m$ edges labeled $\phi_j$
were added in a single step $s$, $s$ appears in $S_j$ with multiplicity $m$.
So each element $s \in S_j$ is an element of $[n]$.
Let $K_n$ be the number of unique vertex pairs seen among edges introduced up until the current step $n$.
Then we may define $C_n$ to be the collection of step sets
across edges that have appeared by step $n$:
\begin{align*}
    C_n = \{ S_1, \ldots, S_{K_n} \}.
\end{align*}
Finally, we can define the
 \emph{step collection sequence}
$C = (C_1, C_2,\ldots)$
as the sequence of $C_n$ for $n=1,2,\ldots$.
Note that it is not clear how to recover the original edge connectivity of the graph
from the step collection sequence, or whether it is possible to modify the sequence (or the labels $\phi_j$) such
that it is easy to recover connectivity while maintaining the (non-trivial) connections to combinatorial
exchangeability provided in \Cref{sec:part,sec:feat,sec:trait} below.
\begin{example}
    Suppose we have the edge sequence
\begin{align*}
	E_1 &= \{\{2,3\}, \{3,6\}\}, \\
    E_2 &= \{\{2,3\}, \{3,6\}\}, \\
    E_3 &= \{\{2,3\}, \{3,6\}, \{6,6\}, \{3,6\}\}, \\
	E_4 &= \{\{2,3\},\{1,4\},\{3,6\}, \{6,6\}, \{3,6\} \},
\end{align*}
with step-augmentation
$$
	E'_4 = \{(\{2,3\},1),(\{1,4\},4),(\{3,6\},1), (\{6,6\},3), (\{3,6\},3)\}
$$
for $E_4$.
Now we label the unique edges in $E'_n$.
Using an order of appearance scheme
\cite{broderick:2013:feature} to index the labels,
$E'_4$ becomes
\begin{align*}
\{(\phi_1,1), (\phi_2,1), (\phi_3,3), (\phi_1,3), (\phi_4,4)\},
\end{align*}
where the labels $\phi_j$ correspond to the four unique vertex pairs:
$\phi_1 = \{3,6\},
\phi_2 = \{2,3\},
\phi_3 = \{6,6\},
\phi_3 = \{1,4\}$.
The step collection
sequence for $C_1,\ldots,C_4$ is
\begin{align*}
    C_1 = \{  \underbrace{\{1\}}_{\phi_1}  \}, \qquad
    C_2 = \{  \underbrace{\{1\}}_{\phi_1}  \}, \qquad
    C_3 = \{  \underbrace{\{1, 3\}}_{\phi_1}, \underbrace{\{3\}}_{\phi_3}  \},
    \qquad
    C_4 = \{  \underbrace{\{1, 3\}}_{\phi_1}, \underbrace{\{3\}}_{\phi_3},
    \underbrace{\{4\}}_{\phi_4}   \}.
\end{align*}
Here each element of $C_n$ is a set corresponding
to one of the four unique labels $\phi_j$ and
contains all step indices up to step $n$
in which an edge with that label was added to
the graph sequence.
\end{example}

To see that the step collection sequence can be interpreted
as a familiar combinatorial object, we recall the following definitions.
A \emph{partition} $C_n$ of $[n]$ is a set
$\{S_1, \ldots, S_{K_n}\}$
whose blocks, or \emph{clusters},
are mutually exclusive, i.e.,
$S_i \cap S_j = \emptyset, i\neq j$,
and exhaustive, i.e.,
$\bigcup_j S_j = [n]$.
Feature allocations relax the definition of partitions
by no longer requiring the blocks to be mutually
exclusive and exhaustive.
A \emph{feature allocation} $C_n$ of $[n]$
is a multiset $\{S_1, \ldots, S_{K_n}\}$
of subsets of $[n]$, such that any datapoint in $[n]$
occurs in finitely many \emph{features} $S_j$ \citep{broderick:2013:feature}.
A \emph{trait allocation} generalizes
the feature allocation where now each
$S_j$, called a \emph{trait},
may itself be a multiset \citep{broderick:2014:posteriors,Campbell16b}.

We see that the step collection $C_n$ can be interpreted
as follows.
If a single edge is added to the graph at each round, $C_n$
is a partition of $[n]$, and the step collection
 sequence is
a projective partition sequence. If at most one edge is added between any pair
of vertices at each step, $C_n$ is a feature allocation of $[n]$, and the step collection sequence
is a projective sequence of feature allocations.
In the most general case, when multiple edges may be added between
any pair of vertices at each step, $C_n$ is a trait allocation of $[n]$, and the
step collection sequence is a projective sequence of trait allocations.

In the following examples, corresponding to
Figure~\ref{fig:partition},
we show different step collection sequences
that correspond to a partition, a feature allocation, and a trait allocation.

\begin{example}[Partition]
Consider the step collection
$C_5 = \{\{1,3\},\{2\},\{4\},\{5\}\}$.
The edges form a partition of the steps.
Here exactly one edge arrives in each step.
\end{example}
\begin{example}[Feature allocation]
Consider the step collection
$C_5 = \{\{1,3\},\{1\},\{1,5\},\{3,4\}\}$.
This step collection forms a feature allocation of the steps.
Thus in this case, there may be multiple \emph{unique} edges arriving in each
step.
\end{example}
\begin{example}[Trait allocation]
In a trait allocation, there may be multiple edges (not necessarily unique) at each step.
Consider the step collection
$C_5 = \{\{1,3,3,3\},\{1\},\{1,5\},\{3\},\{4,4\}\}$.
This collection
forms a trait allocation of the steps, where elements of $C_5$ are now multisets.
\end{example}

\begin{figure}
    \centering
    \begin{subfigure}[b]{0.3\linewidth}
        \centering
        \includegraphics[scale=0.35]{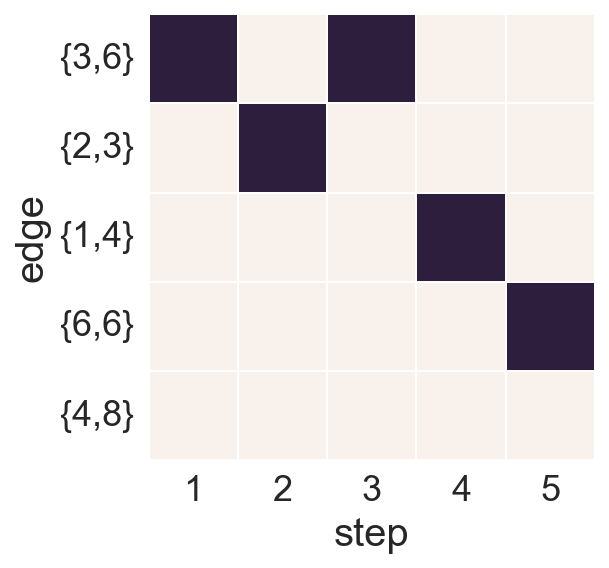}
        \caption{Partition}
    \end{subfigure}
    \begin{subfigure}[b]{0.3\linewidth}
        \centering
        \includegraphics[scale=0.35]{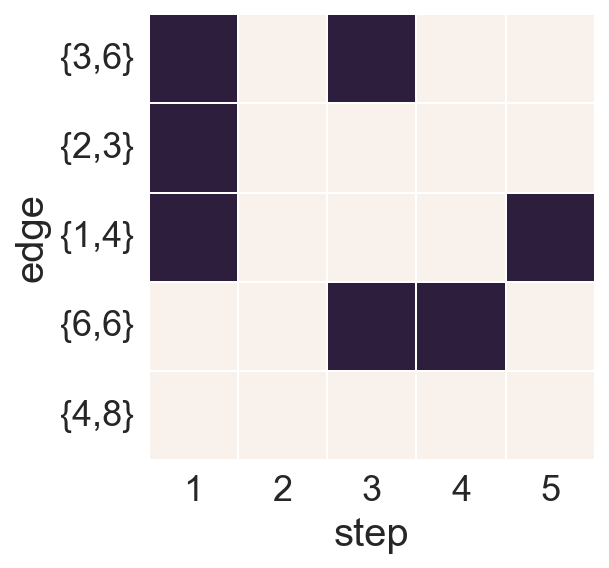}
        \caption{Feature allocation}
    \end{subfigure}
    \begin{subfigure}[b]{0.3\linewidth}
        \centering
        \includegraphics[scale=0.35]{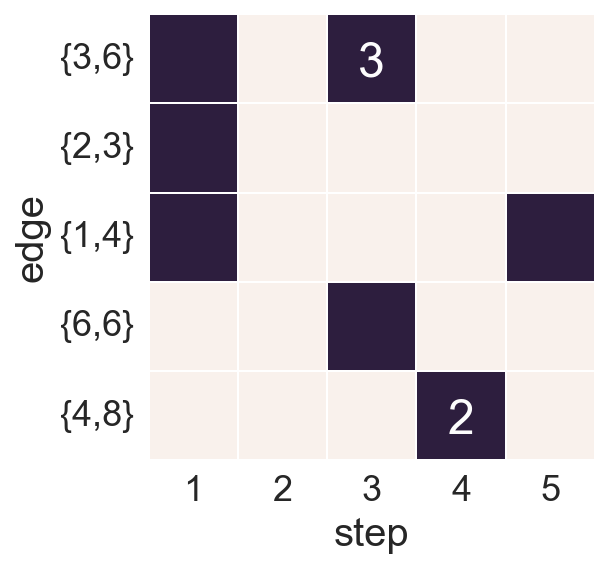}
        \caption{Trait allocation}
    \end{subfigure}
    \caption{Connection of edge-exchangeable graphs with partitions, feature
        allocations, and trait allocations.
        Light blocks represent 0, dark blocks either represent 1 or the
        specified count.
        In a partition, exactly one edge arrives in each step.
        In a feature allocation, multiple edges may arrive at each step, but at most one edge arrives between any two vertices at each step.
        In a trait allocation, there may be multiple edges of any type.
    }
    \label{fig:partition}
\end{figure}

In this section, we have connected certain types of edge-exchangeable graphs
to partitions and feature allocations. In the next two sections, we make use of
known characterizations of these combinatorial objects to characterize edge exchangeability in graphs.

\subsubsection{Partition connection} \label{sec:part}

First consider the connection to partitions.
In this case, suppose that each index in $[n]$ appears exactly once across all
of the subsets of $C_n$. This assumption on $C_n$ is equivalent to assuming that in the original
graph sequence $E_1, E_2, \ldots$, we have that $E_{n+1}$ always has exactly one
more edge than $E_{n}$. In this case, $C_n$ is exactly a
\emph{partition} of $[n]$; that is, $C_n$ is a set of mutually exclusive and
exhaustive subsets of $[n]$. If the edge sequence $(E_n)$ is random,
then $(C_n)$ is random as well.

We say that a partition sequence $C_1, C_2, \ldots$,
where $C_n$ is a (random) partition of $[n]$ and $C_m \subseteq C_n$
for all $m \le n$, is infinitely exchangeable if, for all $n$,
permuting the indices in $[n]$ does not change the distribution
of the (random) partitions \citep{pitman:1995:exchangeable}. Permuting the indices $[n]$
in the partition sequence $(C_m)$
corresponds to permuting the order in which edges are added in our graph
sequence $(E_m)$.
As an example of a model that generates a step collection sequence corresponding
to a partition sequence, consider the frequency model
we introduced in \Cref{sec:generative} where the weights are normalized.
At each step, we choose a single edge according the resulting probability distribution
over pairs of vertices.

Given this connection to exchangeable partitions, the \emph{Kingman paintbox theorem} \citep{kingman:1978:representation}
provides a characterization of edge exchangeability in graph sequences that introduce one edge per step:
in particular, it guarantees that a graph sequence that adds exactly one edge per step
is edge exchangeable if and only if the associated step collection
sequence $(C_n)$ has a Kingman paintbox representation.
An alternate characterization of
edge exchangeability in graph sequences that
introduce one edge per step is provided by
\emph{exchangeable partition probability functions (EPPFs)} \citep{pitman:1995:exchangeable}.
In particular, a graph sequence that introduces one edge per step
is edge-exchangeable if and only if the marginal distribution of
$C_n$ (the step collection at step $n$)
is given by an EPPF
for all $n$.

\subsubsection{Feature allocation connection} \label{sec:feat}

Next we notice that it need not be the case that exactly one edge is added at
each step of the graph sequence, e.g.\ between $E_{n}$ and $E_{n+1}$. If we
allow multiple unique edges at any step, then the step collection $C_n$ is just a set
of subsets of $[n]$, where each subset has at most one of each index in $[n]$.
Suppose that any $m$ belongs to only finitely many subsets in $C_n$ for any $n$.
That is, we suppose that only finitely many edges are added to the graph at any
step. Then $C_n$ is an example of a \emph{feature allocation}
\citep{broderick:2013:feature}. Again, if $(E_n)$ is random, then $(C_n)$ is random as well.

We say that a (random) feature allocation sequence $(C_m)$ is
infinitely exchangeable if, for any $n$, permuting the indices of $[n]$ does not
change the distribution of the (random) feature allocations
\cite{broderick:2013:cluster, broderick:2013:feature}.
Permuting the indices $[n]$ in the sequence $(C_m)$ corresponds to
permuting the steps when edges are added in the edge sequence $(E_m)$.
Consider the following example of
a graph frequency model
that produces a step collection sequence
corresponding to an exchangeable feature allocation.
For $n=1,2,\ldots$, we draw whether the
graph has an edge $\{i,j\}$ at time step $n$ as
Bernoulli with probability
$w_{\{i,j\}} = w_i w_j$.
Thus, in each step, we draw at most one edge per unique vertex pair. But we may draw multiple
edges in the same step.

Similarly to the partition case in \mysec{part}, we can apply known results from feature allocations to characterize edge exchangeability in graph
models of this form. For instance, we know that the \emph{feature paintbox} \cite{broderick:2013:feature,Campbell16b} characterizes distributions
over exchangeable feature allocations (and therefore the step collection sequence for graphs of this form)
just as the Kingman paintbox characterizes distributions
over exchangeable partitions (and therefore the step collection sequence for edge-exchangeable graphs with exactly one new edge per step).

We may also consider feature paintbox distributions with extra structure. For instance, the step collection sequence is said to have
an \emph{exchangeable feature probability function} (EFPF) \citep{broderick:2013:feature} if the probability of each step collection $C_n$ in the sequence
can be expressed as a function only of the total number of steps $n$ and the subset sizes within $C_n$ (i.e.~the edge multiplicities in the graph), and is symmetric in the subset sizes.
As another example, the step collection sequence is said to have a \emph{feature
frequency model} if there exists a (random) sequence of probabilities $(w_j)_{j=1}^\infty$ associated with edges $j=1, 2, \dots$ and a number $\lambda > 0$,
conditioned on which the step collection sequence arises from the graph built by adding edge $j$ at each step independently\footnote{This is conditional independence since the $(w_j)$ may be random.} with probability $w_j$ for all values of $j\in\mathbb{N}$, along with an additional
$\mathrm{Poiss}(\lambda)$ number of edges that never share a vertex with any other edge in the sequence. In other words, the graph is constructed with a graph frequency model as in the main text of the present work
(modulo the aforementioned additional Poisson number of edges).
Theorem 17 (``Equivalence of EFPFs and feature frequency models'') from
\cite{broderick:2013:feature} shows that these two examples are actually equivalent: if the step collection sequence has an EFPF, it has a feature frequency model, and vice versa.

\subsubsection{Further extensions} \label{sec:trait}

Finally, we may consider the case where at every step, any non-negative (finite)
number of edges may be added \emph{and} those edges may have non-trivial
(finite) multiplicity; that is, the multiplicity of any edge at any step can be any
non-negative integer. By contrast, in \mysec{feat}, each unique edge occurred at most
once at each step. In this case, the step collection $C_n$ is a set of subsets of $[n]$.
The subsets need not be unique or exclusive since we assume any number of edges may be
added at any step. And the subsets themselves are multisets since an edge
may be added with some multiplicity at step $n$. We say that $C_n$ is a
\emph{trait allocation},
which we define as a generalization of a feature allocation
where the subsets of $C_n$ are multisets. As above, if $(E_n)$ is random, $(C_n)$ is as well.

We say that a (random) trait allocation sequence $(C_m)$ is infinitely exchangeable
if, for any $n$, permuting the indices of $[n]$ does not change the distribution
of the (random) trait allocation. Here, permuting the indices of $[n]$
corresponds to permuting the steps when edges are added in the edge
sequence $(E_m)$.
A graph frequency model
that generates a step collection sequence as a trait allocation sequence is
the multiple-edge-per-step
frequency model sampling procedure described in
\Cref{sec:generative}. Here, at each step, multiple edges can appear each
with multiplicity potentially greater than 1,
requiring the full generality of a trait allocation sequence.

\citet{Campbell16b} characterize exchangeable trait allocations via, e.g., probability functions and paintboxes
 and thereby provide a
characterization over the corresponding step collection
sequences of such edge-exchangeable graphs.

\subsection{Connections to exchangeability in ordered combinatorial structures}\label{sec:labeledcomb}
As noted earlier, it is not immediately clear how to recover
the connectivity in an edge-exchangeable graph from the step collection sequence, nor how to do so in a way
that preserves non-trivial connections to other exchangeable combinatorial structures.
\citet{Campbell16b} considers an alternative to the step collection sequence in which
the (multi)subsets in the combinatorial structure correspond to \emph{vertices} rather than edges, known as a \emph{vertex allocation}.
This allows for the characterization of edge-exchangeable graphs via the \emph{graph paintbox} using the general theory of
trait allocations, while maintaining an explicit representation of the structure of the graph, i.e., the connection between edges that share a vertex.

If we are willing to eschew the unordered nature of the step collection sequence, and assume that we have an a priori labeling on the vertices,
there is yet another alternative using the \emph{ordered step collection sequence}.
The availability of labeled vertices is often a reasonable assumption in the setting of network data, where the vertices and edges are typically observed directly.
Suppose the vertices are labeled using the natural numbers $1, 2, \dots$.
Then we can use the ordering of the vertex labels to order the vertex pairs in a diagonal manner, i.e.~$\{1, 1\}, \, \{1, 2\}, \, \{2, 2\}, \, \{1, 3\}, \, \{2, 3\}, \dots$. Note that, for the purpose of building this diagonal ordering, we consider the lowest-valued index in each vertex pair first.
We build the step collection sequence $(C_n)$ in the same manner as before, except that each step collection $C_n$ is no longer an unordered collection of subsets;
the subsets derive their order from the vertex pairs they represent. For example, if we observe edges at vertex pairs $\{1, 1\}$ and $\{1, 2\}$ at step 1, and edges at vertex pairs $\{1, 1\}$ and $\{2, 3\}$ at step 2,
then
\begin{align*}
C_1 &= \left( \{1\}, \{1\}, \emptyset, \emptyset, \dots \right)
\shortintertext{and}
C_2 &= \left(\{1, 2\}, \{1\}, \emptyset, \emptyset, \{2\}, \emptyset, \dots\right).
\end{align*}
Since we know the order of the subsets in each $C_n$ as they relate to the vertex pairs in the graph and their connectivity, we can recover the graph sequence from the ordered step
collection sequence $(C_n)$. Exchangeability in an ordered step collection sequence means that the distribution is invariant to permutations of the indices within the subsets (although
the ordering of the subsets themselves cannot be changed). Given this notion of exchangeability, the earlier connections to exchangeable partitions, feature allocations, and trait allocations remain true,
modulo the fact that they must themselves be ordered.
\citet{broderick:2013:feature} provides a paintbox characterization of ordered exchangeable feature allocations, thereby providing characterizations (via the earlier connections to partitions and feature allocations)
of edge-exchangeable graphs that add either one or multiple unique edges per step.
Note that, in these cases, this is a full characterization of edge-exchangeable graphs, by contrast to \Cref{app:step_connections}, where
we provided a characterization only of edge exchangeability in graphs.
We suspect that a similar characterization of edge-exchangeable graphs with multiple (non)unique edges per step
is available by examining characterizations of exchangeable ordered trait allocations.

\section{Proofs}
\label{app:proofs}

The proof of the main theorem in the paper
(Theorem~\ref{thm:asymp}) follows from a collection
of lemmas below.
Lemma~\ref{prop:moments} characterizes the expected number of vertices and edges;
Lemma~\ref{lem:poissmoments} establishes a useful transformation of those expectations;
and Lemma~\ref{lem:poisseqreg} shows that the two sets of expectations are
asymptotically equivalent, so it is enough to consider the transformed expectation.
Lemma~\ref{lem:mainlem} provides the asymptotics of the transformed expectations.
Finally, Lemma~\ref{lem:xvsex} shows that the random sequences converge almost surely to their
expectations, yielding the final result.

\subsection{Preliminaries}

\paragraph{Notation}
We first define the asymptotic notation used
in the main paper and appendix.
We use the notation ``a.s.'' to mean almost surely, or with probability 1.
Let $(X_n)_{n \in \Nats},(Y_n)_{n \in \Nats}$
be two random sequences.
We say that
$X_n \eqas O(Y_n)$ if
$\limsup_{n \rightarrow \infty} \frac{X_n}{Y_n} <\infty$
a.s.,
and that
$X_n \eqas \Omega(Y_n)$ if
$Y_n \eqas O(X_n)$ a.s.
We say that
$X_n \eqas o(Y_n)$ if
$\lim_{n \rightarrow \infty} \frac{X_n}{Y_n} = 0$ a.s.
Lastly,
we say that
$X_n \eqas \Theta(Y_n)$ if
$X_n \eqas O(Y_n)$ and $Y_n \eqas O(X_n)$.

Let $V_n, E_n$
be the respective sets of active vertices
and edges at step $n$ in the multigraph,
and $|V_n|, |E_n|$ be their respective cardinalities,
as defined in the main text.
We use the notation $\bar{V}_n$ and $\bar{E}_n$
to represent these analogous
vertex and edge sets for the binary graph.
Note that $\bar{V}_n$ is the same as $V_n$.

\paragraph{Useful results}
We present two useful theorems
for analyzing expectations involving
random sums of functions of points from
Poisson point processes.
Below, we will apply these theorems repeatedly
to get expectations of graph quantities.
The first theorem is Campbell's theorem,
which is used to compute the moments of functionals of a Poisson process.
We state it below for completeness, and
 refer to \citet[Sec.~3.2]{MR1207584} for details.
\begin{theorem}[Campbell's theorem]\label{thm:campbell}
    Let $\Pi$ be a Poisson point process on $S$ with rate measure $\nu$,
    and let $f:S \rightarrow \Reals$ be measurable. If
    $\int_S \min(|f(x)|,1) \,\nu(dx) <\infty$, then
    \begin{align*}
        \mbe\left(
        \exp\left({c \sum_{x \in \Pi} f(x)}\right)
        \right)
        =
        \exp\left(
        \int_S (\exp( c f(x)) - 1) \,\nu(dx)
        \right)
    \end{align*}
    for any $c \in \mathbb{C}$,
    and furthermore,
    \begin{align*}
        \mbe\left(\sum_{x \in \Pi} f(x) \right) = \int_S f(x) \,\nu(dx).
    \end{align*}
\end{theorem}
The second theorem is a specific form of the
Slivnyak-Mecke theorem,
which is useful for computing the expected sum of
a function of each point $x \in \Pi$ and
$\Pi \setminus \{x\}$ over all points in a Poisson point process $\Pi$.
If each point in $\Pi$ is thought of as relating to a particular vertex in a graph,
the Slivnyak-Mecke theorem allows us to take expectations of the sum (over all possible vertices in the graph) of a function
of each vertex and all its possible edges. For example, it is used below to compute the expected number of active vertices
by taking the expected sum of vertices that have nonzero degree.
We state it below for completeness, and refer to
 \citet[Prop.\ 13.1.VII]{MR2371524} and
\citet[Thm.\ 3.1,Thm.\ 3.2]{baddeley2007spatial} for details.
\begin{theorem}[Slivnyak-Mecke theorem]\label{thm:slivnyakmecke}
    Let $\Pi$ be a Poisson point process on $S$ with rate measure $\nu$,
    and let $f:S \times \Omega \rightarrow \Reals_+$ be measurable.
    Then
    \begin{align*}
        \mbe\left(\sum_{x \in \Pi} f(x, \Pi \setminus \{x\})\right)
        =
        \int_S \mbe\left( f(x, \Pi) \right) \nu(dx).
    \end{align*}
\end{theorem}

\subsection{Graph moments}

In this section, we give the
expected number of vertices
and expected number of edges
for the multi- and binary graph
cases.
We begin by defining the degree $D_i$ of vertex $i$ in the
multigraph and the degree $\bar D_i$ of vertex $i$ in the
 binary graph,
respectively, as
\begin{align}
\label{eq:degree}
D_i = \sum_j M_{\{i, j\}} && \bar D_i = \sum_j \mathbbm{1}\left(M_{\{i, j\}} > 0\right).
\end{align}
Now we present the expected number of edges and vertices. We note that both
the multi- and binary graphs have the same number of (active) vertices, and so
their expectations are the same.
\begin{lemma*}[\ref{prop:moments}, main text]
The expected number of vertices and edges for the multi- and binary graphs are
\begin{align*}
\mbe\left(|\bar V_n|\right) = \mbe\left(|V_n|\right)
      &= \int
      \left[ 1 - \exp\left(-\int 1- (1-wv)^n
      \,\nu(\mathrm{d}v)\right)
      \right]
    \,\nu(\mathrm{d}w),\\
    \mbe\left(|E_n|\right) &= \frac{n}{2}\iint wv \,
    \,\nu(\mathrm{d}w) \,\nu(\mathrm{d}v),
    \\
       \mbe\left(|\bar E_n|\right) &= \frac{1}{2}\iint (1 - (1-wv)^n)
    \, \nu(\mathrm{d}w) \,\nu(\mathrm{d}v).
\end{align*}
\end{lemma*}

\begin{proof}
Using the tower property of conditional expectation and Fubini's theorem,
we have that the expected number of vertices is
\begin{align*}
\mbe \left( |V_n|  \right)
                & = \mbe\left(\mbe \left( \sum_{i} \mathbbm{1}(D_{i} > 0)
                \,\bigg|\, \mathcal{W} \right)\right)
                 = \mbe\left( \sum_{i} \mathbb{P}\left(D_{i} > 0
                \,\bigg|\, \mathcal{W} \right)\right),
\end{align*}
followed by the definition of degree in \Cref{eq:degree} and the
    binomial density,
\begin{align*}
    \hspace{-.3cm}\mbe \left( |V_n|  \right)
                    & = \mbe\left( \sum_i
                \left[1 - \prod_{j} \mathbb{P}\left(M_{\{i,j\}} = 0
                \,|\, \mathcal{W} \right) \right] \right)
                     = \mbe\left( \sum_{w \in \mathcal{W}}
        \left[1 - \hspace{-.4cm} \prod_{v \in \mathcal{W} \setminus \{w\} } \hspace{-.4cm}(1 - wv)^n
        \right] \right).
\end{align*}
Using the Slivnyak-Mecke theorem (\Cref{thm:slivnyakmecke}),
\begin{align*}
\begin{aligned}
\mbe\left(|V_n|\right) & = \int \mbe\left(1 - \prod_{v \in \mathcal{W}} (1 - w v)^n
        \right) \nu(\mathrm{d}w)\\
          &= \int \left[ 1 - \mbe\left(\exp\left(n\sum_{v \in \mathcal{W}} \log
          (1 - w v) \right)\right) \right]
        \nu(\mathrm{d}w),
\end{aligned}
\end{align*}
and finally by Campbell's theorem (\Cref{thm:campbell}) on the inner expectation,
\begin{align*}
\mbe\left(|V_n|\right)
          &= \int \left[ 1 -
          \exp\left(-\int  (1- (1-wv)^n) \,
          \nu(\mathrm{d}v)\right) \right]
          \nu(\mathrm{d}w).
\end{align*}
For the expected number of edges, we can again apply
the tower property and Fubini's theorem followed by repeated applications of
Slivnyak-Mecke to the expectations to get:
\begin{align*}
    \mbe(|E_n|) =
    \mbe\left(\mbe\left( \frac{1}{2} \sum_{i\neq j} M_{\{i,j\}} \big| \mathcal{W}
    \right)\right)
    =
    \frac{1}{2} \int \mbe\left( \sum_{v \in \mathcal{W}} nwv \right) \,\nu(dw)
    =
    \frac{n}{2}
    \int w v \nu(dw) \nu(dv).
\end{align*}
The expected number of edges for the binary case is obtained similarly via
Fubini and Slivnyak-Mecke:
\begin{align*}
    \mbe( |\bar E_n| ) &=
    \mbe\left(\frac{1}{2} \sum_{i\neq j} P(M_{\{i,j\}} > 0 | \mathcal{W} ) \right)
    =
    \frac{1}{2}
    \,\mbe\left( \sum_{w \in \mathcal{W}, v \in \mathcal{W} \setminus \{w\}}
    (1 - (1 - w v)^n )\right)
    \\
     &=
    \frac{1}{2}
    \int \int
    (1 - (1 - w v)^n )
    \,\nu(dw) \,\nu(dv).
\end{align*}
\end{proof}
The asymptotic behavior of these quantities is difficult to derive directly
due to the discreteness of the indices $n$.
Therefore, we rely on a technique
called \emph{Poissonization},
which allows us to bypass this difficulty by
instead considering a continuous
analog of the quantities in order
to get asymptotic behaviors.
Below, we introduce primed notation
$V'_t, E'_t, \bar E'_t, D'_{t,i}$
to represent the Poissonized
quantities for the vertices, multigraph edges, binary edges, and the degree of
a vertex, where the index $t$ now represents a continuous quantity.
These will be defined such that
$V'_N$ has the same asymptotic behavior as $V_N$,
$E'_N$ has the same asymptotic behavior as $E_N$,
and so on.

Given $\mathcal{W}$, let $\Pi_{ij}$ be the Poisson process generated with rate $w_i w_j$ if $i < j$
and rate 0 if $i = j$, and let $\Pi_{ji} = \Pi_{ij}$.
Let $\Pi_i := \bigcup_{j = 1}^\infty \Pi_{ij}$, which is a Poisson process
with rate $u_i := \sum_{j: j \neq i} w_i w_j$ via Poisson process
superposition \citep[Sec.~2.2]{MR1207584}.
If we think of $t$ as continuous time passing, the process $\Pi_{ij}$ represents the times at which new edges are added between vertices $i$ and $j$, and $\Pi_i$ represents the times at which any new edges involving vertex $i$ are added.

Thus, we define the Poissonized degree of vertex $i$ in the multi- and binary graph cases, respectively, to be
a function of the continuous parameter $t > 0$,
\begin{align*}
    D_{t,i}' = |\Pi_{i} \cap [0,t]|,
    \qquad
    \bar D_{t,i}' = \sum_{j} \mathbbm{1}\left(|\Pi_{ij} \cap [0,t]|>0\right).
\end{align*}
We can define the
Poissonized graph quantities of interest using these two quantities:
\begin{align*}
|\bar V'_t| = |V'_t| = \sum_{i} \mathbbm{1}(D'_{t,i} > 0),
    \qquad |E'_t| = \frac{1}{2} \sum_{i=1}^\infty D'_{t,i}, \qquad |\bar E'_t| =
    \frac{1}{2}\sum_{i}\bar D'_{t,i}.
\end{align*}

\begin{lemma}\label{lem:poissmoments}
The expected number of Poissonized vertices and edges for the multi- and binary graphs is
\begin{align*}
\mbe\left(|V'_t|\right) &=
    \int \left[ 1 - \exp\left( -\int (1-e^{-t wv})
    \,\nu(dv)\right) \right] \,\nu(dw)\\
    \mbe\left(|E'_t|\right) &=
    \frac{t}{2}  \iint w v
    \,\nu(\mathrm{d}w)\,\nu(\mathrm{d}v)
\\
    \mbe\left(|\bar E'_t|\right) &=
    \frac{1}{2}\iint
    (1-\exp(-twv))
    \,\nu(\mathrm{d}w)\,\nu(\mathrm{d}v).
\end{align*}
\end{lemma}

\begin{proof}
For the expected number of Poissonized vertices, we
apply the tower property and Fubini's theorem to get
\begin{align*}
    \mbe \left( |V'_t|  \right)
    &= \mbe\left(\mbe \left( \sum_i \mathbbm{1}(D'_{t,i} > 0)
    \,\Bigg|\, \mathcal{W} \right) \right)
    = \mbe\left(  \sum_i
    1-\mathbb{P} \left( D_{t,i} = 0
    \,|\, \mathcal{W}\right) \right).
\end{align*}
    Using the fact that $D'_{t,i} | \mathcal{W}$ is Poisson-distributed,
\begin{align*}
    \mbe \left( |V'_t|  \right)
    &= \mbe\left(  \sum_i
    1 - \exp\left(-tu_i\right) \right)
    = \mbe\left(  \sum_{w \in \mathcal{W}}
    1 -
    \exp\left(-t w \sum_{v \in \mathcal{W}\setminus\{w\}} v\right)
    \right).
\end{align*}
Finally, by the Slivnyak-Mecke theorem and Campbell's theorem,
\begin{align*}
    \mbe \left( |V'_t|  \right)
    &=
    \int
    \mbe\left(1 -
    \exp\left(-t w \sum_{v \in \mathcal{W}} v\right)\right) \,\nu(dw)
    \\
    &=  \int
    \left[ 1 -
    \exp\left( \int (e^{-t wv}-1)
    \,\nu(dv)\right) \right] \,\nu(dw).
\end{align*}
For the expected number of Poissonized edges,
after applying Fubini's theorem and Slivnyak-Mecke
we have
\begin{align*}
    \mbe\left( |E'_t| \right) &=
    \mbe\left( \frac{1}{2} \sum_i D'_{t,i} \right) =
    \mbe\left( \frac{1}{2} \sum_i \mbe\left(D'_{t,i} | \mathcal{W}\right) \right)
    \\
    &=
    \mbe\left( \frac{1}{2} \sum_i u_i \right)
    =
    \mbe\left( \frac{1}{2} \sum_{w \in \mathcal{W}, v \in \mathcal{W} \setminus \{w\} } wv \right)
    \\
    &=
     \frac{1}{2}  \int\int wv    \,\nu(dw) \,\nu(dv).
\end{align*}
For the expected number of Poissonized edges in the binary case, noting that
$|\Pi_{ij} \cap [0,t]|$  is Poisson-distributed with rate $t w_i w_j$, and
    applying Fubini's theorem and Slivnyak-Mecke, we have:
\begin{align*}
    \mbe(|\bar E'_t|)
    &=
    \mbe\left( \mbe\left( \sum_i \bar D'_{t,i} | \mathcal{W}  \right)\right)
    =
    \mbe\left( \sum_{w \in \mathcal{W}, v \in \mathcal{W} \setminus \{w\}}
    (1 - \exp(-twv))\right)
    \\
    &=
    \int\int   (1 - \exp(-twv)) \,\nu(dw)\,\nu(dv).
\end{align*}
\end{proof}

\subsection{Asymptotics}

We have defined the expected number
of vertices and edges for the multigraph and
binary graph cases
(Lemma~\ref{prop:moments})
and presented the Poissonized version
of these expectations
(Lemma~\ref{lem:poissmoments}).
We now show in Lemma~\ref{lem:poisseqreg}
that the expected graph quantities and their
Poissonized expectations
are asymptotically equivalent.
\begin{lemma}\label{lem:poisseqreg}
The Poissonized expectations
for the number of vertices and the number of edges
in the multi- and binary graphs
are asymptotically equivalent to the original
expectations; i.e.,
as $n \rightarrow \infty$,
\begin{align*}
 \mbe\left(|V'_n|\right) \sim \mbe\left(|V_n|\right),
    \\
    \mbe\left(|E'_n|\right) \sim \mbe\left(|E_n|\right),
    \\
    \mbe\left(|\bar
    E'_n|\right) \sim \mbe\left(|\bar E_n|\right).
\end{align*}
\end{lemma}
\begin{proof}
For the vertices, we have
\begin{align*}
\mbe\left(|V_n|-|V'_n|\right) &=
    \!\int \!\!
    \left[
    \exp\left(-\!\textstyle\int (1-e^{-nwv})
    \,\nu(\mathrm{d}v)\right)
    - \exp\left( -\!\textstyle\int (1-(1-wv)^n)
    \,\nu(\mathrm{d}v)\right)
    \right]
    \,\nu(\mathrm{d}w) .
\end{align*}
Using the elementary inequalities
\begin{align*}
0 &\leq e^{-nx} - (1-x)^n \leq nx^2e^{-nx}, \qquad x\in[0, 1], \, \, n>0\\
0 &\leq e^{-a} - e^{-b} \leq b-a, \qquad 0\leq a \leq b,
\end{align*}
we have
\begin{align}
0\leq \mbe\left(|V_n|-|V'_n|\right) &\leq \iint n(wv)^2 e^{-nwv}
    \,\nu(\mathrm{d}v)\,\nu(\mathrm{d}w).\label{eq:vertexpdiff}
\end{align}
Finally, note that
\begin{align*}
 \forall n> 0, \forall w, v \in [0, 1], \quad n w v e^{-n wv} \leq e^{-1}
\end{align*}
and
\begin{align*}
 \iint e^{-1} wv
    \,\nu(\mathrm{d}w)
    \,\nu(\mathrm{d}v) =
    e^{-1}\left(\int w
    \,\nu(\mathrm{d}w)\right)^2 < \infty.
\end{align*}
Therefore by Lebesgue dominated convergence,
\begin{align*}
0 \leq \lim_{n\to\infty} \mbe\left(|V_n|-|V'_n|\right) &\leq  \iint
    \lim_{n\to\infty} n(wv)^2 e^{-nwv}
    \,\nu(\mathrm{d}v)\,\nu(\mathrm{d}w) = 0,
\end{align*}
so we have that
    $\lim_{n\to\infty} \mbe\left(|V_n|-|V'_n|\right) = 0$.
    Since $\mbe(|V_n|)$, $\mbe(|V'_n|)$ are monotonically increasing by inspection,
    $\mbe(|V_n|)\sim\mbe(|V'_n|)$, $n\to\infty$,
as required.

For the binary graph edges,
\begin{align*}
\mbe\left(|\bar E_n|-|\bar E'_n|\right) &= \frac{1}{2}
    \iint
    (\exp(-nwv) - (1-wv)^n)
    \,\nu(\mathrm{d}v) \,\nu(\mathrm{d}w).
\end{align*}
Using the earlier elementary inequalities,
\begin{align*}
0 \leq \mbe\left(|\bar E_n|-|\bar E'_n|\right)
    &= \frac{1}{2}
    \iint n(wv)^2e^{-nwv}
    \,\nu(\mathrm{d}v)\,\nu(\mathrm{d}w).
\end{align*}
This is (modulo the constant factor of $\nicefrac{1}{2}$) the exact expression in \Cref{eq:vertexpdiff}. Therefore, the same analysis
can be performed, and the result holds.

For multigraph edges,
\begin{align*}
\mbe\left(|E_n|-|E'_n|\right) &= \frac{n}{2}
    \iint
    (wv - wv)
    \,\nu(\mathrm{d}v) \,\nu(\mathrm{d}w) = 0,
\end{align*}
so $\mbe\left(|E_n|\right) \sim \mbe\left(|E'_n|\right)$, $n\to\infty$.
\end{proof}

Lemma~\ref{lem:poisseqreg} allows us to analyze the asymptotics
of the Poissonized expectations and apply the result directly
to the asymptotics of the original graph quantities.
To achieve the desired asymptotics for the Poissonized expectations, we will make a further
assumption on the rate measure $\nu$
generating the vertex weights in \Cref{eq:graphmodel}.
Namely, we assume that the tails of $\nu$ decay at a rate that will yield the
appropriate weight decay in the weights $(w_j)$---and
thereby the appropriate decay in vertex creation to finally yield sparsity in the graph itself.
In particular,
the tail of a measure $\nu$ is said to be \emph{regularly varying} if
there exists a function $\ell : \mathbb{R}_+ \to \mathbb{R}_+$ and $\alpha \in (0, 1)$ such that
\begin{align}
\int_x^1\nu(\mathrm{d} w) &\sim x^{-\alpha}\ell(x^{-1}), \quad x\to 0,
&&
    \forall \, c >0, \,\,
    \lim_{x \to \infty} \frac{\ell(cx)}{\ell(x)} = 1. \label{eq:regvar}
\end{align}
The condition on the function $\ell$ is equivalent to saying that $\ell$ is \emph{slowly varying}.
For additional details on regular and slow variation,
see \citet[VIII.8]{MR0270403}. An important equivalent formulation of \Cref{eq:regvar} that we will use
in our following proof of the asymptotics
of Poissonized expectations is provided by Lemma~\ref{gnedin13}
(see \citet[Prop.~13]{MR2318403} and \citet[Prop.~6.1]{MR2934958}
for the proof).
\begin{lemma}[\citet{MR2318403,MR2934958}]
    \label{gnedin13}
    The tail of $\nu$ is regularly varying
    iff there exists a function $\ell : \mathbb{R}_+ \to \mathbb{R}_+$ and $\alpha \in (0, 1)$ such that
    \begin{align}
        \int_0^x u \nu(du) \sim x^{1-\alpha}
        \ell(x^{-1}), \quad x \to 0, && \forall \, c >0, \,\,
    \lim_{x \to \infty} \frac{\ell(cx)}{\ell(x)} = 1.
    \end{align}
\end{lemma}
Lemma~\ref{gnedin13} is often easier to use than \Cref{eq:regvar} when checking whether a particular measure $\nu$ has a
regularly varying tail.
For example, for the three-parameter beta process, we have
\begin{align*}
    \int_0^x  u \nu(du)
    &= \gamma \frac{\Gamma(1+\beta)}{\Gamma(1-\alpha)\Gamma(\beta+\alpha)}
    \int_0^x u^{-\alpha} (1-u)^{\beta + \alpha - 1} du \\
    &\sim \gamma \frac{\Gamma(1+\beta)}{\Gamma(1-\alpha)\Gamma(\beta+\alpha)}
    \int_0^x u^{-\alpha}  du, \quad x \downarrow 0
    \\
    &= \gamma \frac{\Gamma(1+\beta)}{\Gamma(1-\alpha)\Gamma(\beta+\alpha)}
    \frac{1}{1-\alpha} x^{1-\alpha},
\end{align*}
so the tail of $\nu$ is regularly varying when the discount parameter $\alpha$ satisfies $\alpha\in(0,1)$ with
$\ell(x^{-1})$ equal to the constant function
\begin{align}
    \label{eq-3bp-const}
    \ell(x^{-1}) = \frac{\gamma}{1-\alpha}
    \frac{\Gamma(1+\beta)}{\Gamma(1-\alpha)\Gamma(\beta+\alpha)}.
\end{align}
Note that the two-parameter beta process does not exhibit this behavior (since in this case, $\alpha = 0$).

Given the two formulations of a measure $\nu$ with a regularly varying tail above,
we are ready to characterize the asymptotics of the earlier Poissonized expectations.
\begin{lemma}\label{lem:mainlem}
If the tail of $\nu$ is regularly varying as per \Cref{eq:regvar}, then as $n\to\infty$,
\begin{align*}
\mbe\left(|V'_n|\right) =\Theta(n^\alpha\ell(n)), && \mbe\left(|E'_n|\right) = \Theta(n), && \mbe\left(|\bar E'_n|\right) = O\left(\ell(\sqrt{n})\min\left(n^{\frac{1+\alpha}{2}},\ell(n)n^{\frac{3\alpha}{2}}\right)\right).
\end{align*}
\end{lemma}

\begin{proof}
Throughout this proof we use $c$ to denote a constant; the precise value of $c$ changes but is immaterial.
We also define the tail of $\nu$ as $\bar\nu(x) := \int_x^1\nu(\mathrm{d}w)$, for notational brevity.
Furthermore, recall that we assume the rate measure
$\nu$ satisfies $\int w \nu(dw) <\infty$.

We first examine the expected number of Poissonized vertices,
\begin{align*}
    \mbe\left(|V'_n|\right) &= \int \left[1 - \exp\left( -\int (1-e^{-n wv})
    \nu(dv)\right)\right] \,\nu(dw),
\end{align*}
by splitting the integral into two parts.
    First, by the assumption that the tail of $\nu$ is regularly varying,
\begin{align}
    \label{pois-verts-upper1}
    \int_{n^{-1}}^1 \left[1 - \exp\left( -\int (1-e^{-n wv}) \nu(dv)\right)\right] \,\nu(dw) & \leq \int_{n^{-1}}^1 \nu(dw) \sim c n^{\alpha} \ell(n).
\end{align}
Next, we upper bound the integral term
\begin{align}
    \int_{0}^{n^{-1}} \left[1 - \exp\left( -\int (1-e^{-n wv})
    \nu(dv)\right)\right] \,\nu(dw) & \leq \int_{0}^{n^{-1}}\int (1-e^{-nwv})
    \nu(\mathrm{d}v)\nu(\mathrm{d}w) \notag\\
&\leq \int_{0}^{n^{-1}}\int nwv\nu(\mathrm{d}v)\nu(\mathrm{d}w) \notag \\
    &\leq \left(\int v\nu(\mathrm{d}v)\right) n
    \int_{0}^{n^{-1}}w\nu(\mathrm{d}w) \notag\\
&\sim c n^\alpha \ell(n),
\label{pois-verts-upper2}
\end{align}
where the asymptotic behavior in the last line follows from
    Lemma~\ref{gnedin13}.
Thus,
    combining the upper  bounds on \Cref{pois-verts-upper1}
    and \Cref{pois-verts-upper2} gives the bound for the entire integral:
$\mbe\left(|V'_n|\right) = O(n^\alpha \ell(n))$.

Now we bound the integral below:
\begin{align*}
    \lefteqn{ \int_{n^{-1}}^1
    \left[ 1 - \exp\left( -\int (1-e^{-n wv}) \nu(dv)\right)\right] \,\nu(dw) } \\
    & \geq \int_{n^{-1}}^1
    \left[ 1 - \exp\left(-\int (1-e^{- v} ) \nu(dv)\right) \right] \,\nu(dw) \\
    &=\left(\int_{n^{-1}}^1\nu(\mathrm{d}w)\right)\left(1 - \exp\left( -\int (1-e^{- v})\nu(dv)\right)\right)\\
&\sim c n^\alpha \ell(n),
\end{align*}
where the last line follows from the assumption that the tail of $\nu$ is regularly varying.
The second piece of the integral on $[0, n^{-1}]$ is bounded below by 0,
and in combination, we have that $n^\alpha \ell(n) = O\left(\mbe\left(|V'_n|\right)\right)$.
Now combining this with the previous upper bound result, we have
$\mbe\left(|V'_n|\right) = \Theta(n^\alpha \ell(n))$.

The expected number of Poissonized multigraph edges satisfies $\mbe\left(E'_n\right) =\Theta(n)$,
    since
\begin{align*}
    \mbe(|E_n'|) = \frac{n}{2} \iint wv \nu(dw)\nu(dv)
    =
    \frac{n}{2}
    \int w\nu(dw)
    \int v\nu(dv)
    =
    \frac{c^2 }{2} n.
\end{align*}

For the Poissonized binary graph edges, we split the integral into two pieces.
We first upper bound the integral on the interval $[0, n^{-\nicefrac{1}{2}}]$ and apply \Cref{gnedin13} to get the following asymptotic
behavior:
\begin{align*}
\frac{1}{2}\int_{0}^{n^{-1/2}}\!\!\!\!\!
    \int (1-\exp\left(-nwv\right))
    \,\nu(\mathrm{d}w) \,\nu(\mathrm{d}v)
    &\leq \frac{1}{2}\int_{0}^{n^{-1/2}}
    \!\!\!\!\!
    \int nwv
    \,\nu(\mathrm{d}w) \,\nu(\mathrm{d}v) \\
&= \frac{n}{2}\left(\int w\nu(\mathrm{d}w)\right)\int_{0}^{n^{-1/2}}v \nu(\mathrm{d}v) \\
    &\sim c n (n^{-1/2})^{1-\alpha}\ell(n^{1/2})
\\
    &= cn^{\frac{1+\alpha}{2}}\ell(n^{1/2}).
\end{align*}
We then bound the second portion on the interval $[n^{-\nicefrac{1}{2}}, 1]$ by
linearizing at $v = n^{-1/2}$.
Using integration by parts and an
    Abelian theorem  \cite[Sec.~XIII.5, Thm.~4]{MR0270403}
    for the Laplace transform, for some constants $b, d > 0$, we have
\begin{align*}
\lefteqn{ \frac{1}{2}
    \int_{n^{-1/2}}^{1} \!
    \int (1-\exp\left(-nwv\right))
    \,\nu(\mathrm{d}w) \,\nu(\mathrm{d}v) } \\
&\leq \frac{1}{2}\int_{n^{-1/2}}^{1} \!
    \int \left(1-\exp\left(-{n}^{1/2} w\right)
    + nw\exp\left(-{n}^{1/2} w\right)
    (v-{n}^{-1/2})\right)
    \,\nu(\mathrm{d}w) \,\nu(\mathrm{d}v) \\
&= \frac{1}{2}
    \left(\int_{n^{-1/2}}^1\hspace{-.5cm} \nu(\mathrm{d}v)\right)
    \!\!\int\!\! {n}^{1/2} \exp(-n^{1/2} w)
    \,\bar\nu(w)\,\mathrm{d}w
    \\
    &\qquad +\frac{1}{2}\int_{n^{-1/2}}^1 \hspace{-.5cm} (nv-{n}^{1/2})\,\nu(\mathrm{d}v)
    \!\int\! w\exp(-{n}^{1/2}w)
    \,\nu(\mathrm{d}w) \\
    &\sim b n^{\alpha}\ell^2({n}^{1/2})
+\frac{1}{2}
    \int_{0}^1v
    \,\nu(\mathrm{d}v) \,{n}^{1/2} \!\int
    {n}^{1/2}
    \left(\exp(-{n}^{1/2} w)
    - n^{1/2}
    w\exp\left(-{n}^{1/2} w\right)\right)
    \,\bar\nu(w)\,\mathrm{d}w \\
    &\leq b n^{\alpha}\ell^2({n}^{1/2})
    +\frac{1}{2}\int_{0}^1v
    \,\nu(\mathrm{d}v) \,\,{n}^{1/2}
    \!\int {n}^{1/2}
    \exp(-{n}^{1/2} w) \,\bar\nu(w)\,\mathrm{d}w \\
    &\sim b n^{\alpha}\ell^2({n}^{1/2})
+d {n}^{1/2} {n}^{\alpha/2}
    \ell({n}^{1/2})\\
    &= O(n^{\frac{1+\alpha}{2}}\ell({n}^{1/2})).
\end{align*}
Therefore we have that
$\mbe\left(|\bar E'_n|\right) = O(n^{\frac{1+\alpha}{2}}\ell({n}^{1/2}))$.

To get the other bound, we split the integral into three pieces. First,
\begin{align*}
\lefteqn{ \frac{1}{2}\int_{0}^{n^{-1}}
    \int (1-\exp\left(-nwv\right))
    \,\nu(\mathrm{d}w)\,\nu(\mathrm{d}v) } \\
&\leq \frac{1}{2}\int_{0}^{n^{-1}}
    \int nwv
    \,\nu(\mathrm{d}w)\,\nu(\mathrm{d}v) \\
&= \frac{n}{2}\left(\int w
    \,\nu(\mathrm{d}w)\right)
    \int_{0}^{n^{-1}}v \,\nu(\mathrm{d}v) \\
&\sim c n (n^{-1})^{1-\alpha}\ell(n) = cn^{\alpha}\ell(n).
\end{align*}
Next, integration by parts yields
\begin{align*}
\lefteqn{ \frac{1}{2}
    \int_{n^{-1/2}}^{1}
    \int (1-\exp\left(-nwv\right))
    \,\nu(\mathrm{d}w)\,\nu(\mathrm{d}v) } \\
    &\leq \frac{1}{2}\int_{n^{-1/2}}^{1}\int (1-\exp(-nw))
    \,\nu(\mathrm{d}w) \,\nu(\mathrm{d}v) \\
&= \frac{1}{2}\left(\int_{n^{-1/2}}^1
    \,\nu(\mathrm{d}v)\right)\int n\exp(-nw)
    \,\bar\nu(w)\,\mathrm{d}w \\
    &\sim c \left(n^{-1/2}\right)^{-\alpha}
    \ell({n}^{1/2}) n^{\alpha}\ell(n)\\
    &= cn^{\frac{3\alpha}{2}}\ell(n)\ell({n}^{1/2}).
\end{align*}
Finally, integration by parts yields the final upper bound
\begin{align*}
\lefteqn{ \frac{1}{2}
    \int_{n^{-1}}^{n^{-1/2}} \!\!\!\int
    (1-\exp\left(-nwv\right))
    \,\nu(\mathrm{d}w) \,\nu(\mathrm{d}v) } \\
&\leq \frac{1}{2}\int_{n^{-1}}^{n^{-1/2}}
    \!\!\!\int (1-\exp(-{n}^{1/2} w))
    \,\nu(\mathrm{d}w)\,\nu(\mathrm{d}v) \\
&= \frac{1}{2}\left(\int_{n^{-1}}^{n^{-1/2}}
    \,\nu(\mathrm{d}v)\right)
    \int n^{1/2} \exp\left(-{n}^{1/2} w\right)
    \,\bar\nu(w) \,\mathrm{d}w\\
&\sim \left(c_1 n^\alpha\ell(n) -
    c_2n^{\frac{\alpha}{2}}\ell({n}^{1/2})\right)
    \left(c_3{n}^{\alpha/2} \ell({n}^{1/2})\right) \\
    &\sim cn^{\frac{3\alpha}{2}}\ell(n)\ell({n}^{1/2}).
\end{align*}
Therefore $\mbe\left(|\bar E'_n|\right) =
O(\ell(n)\ell({n}^{1/2}) \,n^{\frac{3\alpha}{2}})$.
\end{proof}

Finally, we show that $|E_n|$, $|\bar E_n|$, and $|V_n|$ are asymptotically equivalent to their expectations almost surely;
thus, the asymptotic results for the expectation sequences applies to the random sequences.
\begin{lemma*}[\ref{lem:xvsex}, main text]
The number of edges and vertices for both the multi- and binary graphs satisfy
\begin{align*}
 |E_n| \overset{\text{a.s.}}{\sim} \mbe\left(|E_n|\right), \qquad |\bar E_n| \overset{\text{a.s.}}{\sim} \mbe\left(|\bar E_n|\right) \qquad |\bar V_n| = |V_n| \overset{\text{a.s.}}{\sim} \mbe\left(|V_n|\right), \qquad n\to \infty.
\end{align*}
\end{lemma*}

\begin{proof}
We use $X_n$ to refer to $|E_n|$, $|\bar E_n|$, or $|V_n|$, since the proof technique is the same for all.
Since we need to show $X_n/\mbe\left(X_n\right) \overset{\text{a.s.}}{\to} 1$,
by the Borel-Cantelli lemma it is sufficient to show that for any $\epsilon > 0$,
\begin{align*}
\sum_{n} P(\left|X_n-\mbe\left(X_n\right)\right| > \epsilon \mbe\left(X_n\right)) < \infty.
\end{align*}
By the union bound, and the fact that $X_n$ can be expressed as a countable sum of
indicators combined with the note after Theorem 4 in \citet{freedman1973},
\begin{align*}
\lefteqn{ P(\left|X_n-\mbe\left(X_n\right)\right| > \epsilon \mbe\left(X_n\right)) } \\
& \leq P(X_n > (1+\epsilon) \mbe\left(X_n\right)) + P(X_n < (1-\epsilon)\mbe\left(X_n\right))\\
&\leq  2 \exp\left(-\frac{\epsilon^2 \mbe\left(X_n\right)}{2} \right).
\end{align*}
Since $\mbe(X_n) \geq n^\beta$ for some $\beta > 0$, the expression is summable and the result holds.
\end{proof}

Combining the results of Lemmas \ref{lem:xvsex}, \ref{lem:poisseqreg}, and \ref{lem:mainlem}
gives us the main theorem, which we state here for completeness.

\begin{theorem*}[\ref{thm:asymp}, main text]
    If the tail of $\nu$ is regularly varying as per \Cref{eq:regvar}, then as $n\to\infty$,
\begin{align*}
\hspace{-.2cm}|V_n| \eqas \Theta(n^\alpha\ell(n)),
    && |E_n| \eqas \Theta(n),
    && |\bar E_n|
    \eqas O\left(\ell({n}^{1/2})\min\left(n^{\frac{1+\alpha}{2}},\ell(n)n^{\frac{3\alpha}{2}}\right)\right).
\end{align*}
\end{theorem*}

\begin{remark}
Finally, to conclude that there exists a class of sparse, edge-exchangeable graphs, we
examine the asymptotics from this result in more detail.
    In the multigraph case, we see that the number of vertices increases at the same rate as
    $n^{\alpha}\ell(n)$, and the number of edges increases linearly in $n$.
    So $|E_n|$ grows at the same rate as $|V_n|^{1/\alpha}\ell(n)^{-1/\alpha}$.
    When $\alpha \in (\nicefrac{1}{2}, 1)$, the exponent $1/\alpha$ lies
    in the range $(1,2)$,
    and thus this parameter range for $\alpha$ results in sparse
    graph sequences.
    For binary graphs, the number of edges
    $|\bar E_n|$ grows at a rate that is bounded by
    $\ell(\sqrt{n})\min\left\{ |V_n|^{\frac{1+\alpha}{2\alpha}} \ell(n)^{-\frac{1+\alpha}{2\alpha}}, |V_n|^{\frac{3}{2}} \ell(n)^{-\frac{1}{2}}\right\}$.
    Since $\min\left\{\frac{1+\alpha}{2\alpha}, \frac{3}{2}\right\} \leq \nicefrac{3}{2} < 2$, binary graphs are sparse for any $\alpha\in(0, 1)$.
    Note that $\ell(n)$ does not affect the growth rate
    throughout since it is a slowly-varying function;
    i.e., for all $c>0$, $\ell(cn) \sim \ell(n)$.
    For the three-parameter beta process, which we examined
    in our simulations,
    the function $\ell$ is a constant function, as in
    \Cref{eq-3bp-const}.
    \label{proof-remark}
\end{remark}

    We have shown that edge exchangeability admits sparse graphs by
    proving the existence of sparse graph sequences in a wide subclass of graph frequency models:
    those frequency models with
    weights generated from Poisson point processes whose rate measures have power law tails.
    Notably, we have shown the existence of a range of sparse and dense
    behavior in this wide class of graph frequency models, as desired.

\newpage
\small
\bibliographystyle{plainnat-mod}
\bibliography{sources}

\begin{thebibliography}{48}
\bibitem[Freedman(1973)]{freedman1973}
D.~Freedman.
\newblock Another note on the {{B}}orel-{{C}}antelli lemma and the strong law,
  with the {{P}}oisson approximation as a by-product.
\newblock \emph{The Annals of Probability}, 1\penalty0 (6):\penalty0 910--925,
  1973.
\end{thebibliography}

\end{document}